\documentclass[sigconf]{aamas} 

\usepackage{balance} 
\usepackage[utf8]{inputenc} 
\usepackage[T1]{fontenc}    
\usepackage{hyperref}       
\usepackage{url}            
\usepackage{booktabs}       
\usepackage{amsfonts}       
\usepackage{nicefrac}       
\usepackage{microtype}      
\usepackage{xcolor}         

\usepackage{subcaption}
\usepackage[]{graphicx}
\usepackage{bbm}
\usepackage{amsmath}
\usepackage{mathtools,algorithm}
\usepackage[noend]{algpseudocode}
\usepackage[symbol]{footmisc}
\usepackage{graphicx}
\graphicspath{ {./figures/} }

\usepackage{titlesec}

\theoremstyle{proposition}

\theoremstyle{theorem}
\newtheorem{theorem}{Theorem}[section]

\theoremstyle{definition}
\newtheorem{definition}{Definition}[section]

\theoremstyle{lemma}
\newtheorem{lemma}{Lemma}[section]

\theoremstyle{fact}
\newtheorem{fact}{Fact}[section]

\newcommand{\E}{\mathbb{E}}
\usepackage{enumitem}
\usepackage[font={small}]{caption}

\setcopyright{ifaamas}
\acmConference[AAMAS '22]{Proc.\@ of the 21st International Conference
on Autonomous Agents and Multiagent Systems (AAMAS 2022)}{May 9--13, 2022}
{Online}{P.~Faliszewski, V.~Mascardi, C.~Pelachaud,
M.E.~Taylor (eds.)}
\copyrightyear{2022}
\acmYear{2022}
\acmDOI{}
\acmPrice{}
\acmISBN{}

\acmSubmissionID{579}

\title{Centralized Model and Exploration Policy for Multi-Agent RL}

\author{Qizhen Zhang}
\affiliation{
  \institution{University of Toronto, Vector Institute}
  \city{Toronto}
   \country{Canada}}
  \email{qizhen@cs.toronto.edu}

\author{Chris Lu}
\affiliation{
  \institution{University of Oxford}
  \city{Oxford}
  \country{UK}}
  \email{christopher.lu@exeter.ox.ac.uk}

\author{Animesh Garg}
\affiliation{
\institution{University of Toronto, Vector Institute, NVIDIA}
\city{Toronto}
\country{Canada}}
\email{garg@cs.toronto.edu}

\author{Jakob Foerster}
\affiliation{
  \institution{University of Oxford}
  \city{Oxford}
  \country{UK}}
  \email{jakob.foerster@eng.ox.ac.uk}

\begin{abstract}

Reinforcement learning (RL) in partially observable, fully cooperative multi-agent settings (Dec-POMDPs) can in principle be used to address many real-world challenges such as controlling a swarm of rescue robots or a team of quadcopters. However, Dec-POMDPs are significantly harder to solve than single-agent problems, with the former being \textit{NEXP-complete} and the latter, MDPs, being just P-complete. Hence, current RL algorithms for Dec-POMDPs suffer from poor sample complexity, which greatly reduces their applicability to practical problems where environment interaction is costly. Our key insight is that using just a \textit{polynomial} number of samples, one can learn a \emph{centralized} model that generalizes across different policies. We can then optimize the policy within the learned model instead of the true system, without requiring additional environment interactions. We also learn a centralized exploration policy within our model that learns to collect additional data in state-action regions with high model uncertainty.  We empirically evaluate the proposed model-based algorithm, MARCO\footnote{MARCO is short for \textbf{M}ulti-\textbf{A}gent \textbf{R}L with \textbf{C}entralized M\textbf{o}dels and Exploration \\ Code available at \texttt{https://github.com/irenezhang30/MARCO/}.}, in three cooperative communication tasks, where it improves sample efficiency by up to 20x. Finally, to investigate the theoretical sample complexity, we adapt an existing model-based method for tabular MDPs to Dec-POMDPs, and prove that it achieves polynomial sample complexity.
\end{abstract}


\newcommand{\BibTeX}{\rm B\kern-.05em{\sc i\kern-.025em b}\kern-.08em\TeX}

\begin{document}


\pagestyle{fancy}
\fancyhead{}


\maketitle 


\section{Introduction}
\label{sec:intro}
Decentralized partially observable Markov Decision Processes (Dec-POMDPs) describe many real-world problems~\citep{liu2017learning, ragi2014decentralized}, but they are significantly harder to solve than Markov Decision Processes (MDPs). This is because the policy of one agent in Dec-POMDPs effectively serves as the \emph{observation function} of other agents and, hence, agents need to explore over \emph{policies} rather than actions. As a consequence, solving Dec-POMDPs involves searching through the space of tuples of decentralized policies that map individual action-observation histories to actions. This space is double exponential ~\cite{oliehoek2008cross}: \begin{equation}
O\left[\left(\left|A\right|^{\frac{\left(\left|O\right|^{H}-1\right)}{\left|O\right|-1}}\right)^{N}\right],
\end{equation} where $|O|$ is size of the observation space, $|A|$ is size of the individual action space, $H$ is the horizon, and $N$ is the number of agents. Since finding an optimal solution is doubly exponential in the horizon, the problem falls into a class called non-deterministic exponential (NEXP)-Complete~\cite{bernstein2002complexity}.

Solving problems in this class is much harder than solving MDPs, which is just P-complete~\cite{papadimitriou1987complexity}. Indeed, current deep multi-agent RL algorithms for learning approximate solutions in Dec-POMDPs, which mostly extend model-free approaches such as independent learning~\cite{tan1993multi}, suffer from high sample complexity \citep{bard2020hanabi, samvelyan19smac}. This limits their applicability in real world problems and other settings where interactions with the environment are costly. 

To address the problem of high environment sample complexity, MARCO uses a model-based approach. This is motivated by two reasons: 1) We take advantage of centralized training in Dec-POMDPs to learn a model of the environment that generalizes across different policies in just a polynomial number of samples (in the joint-action space and state space) like in single agent RL~\cite{strehl2009reinforcement}. In contrast, as mentioned above, the sample complexity for learning an optimal policy is much larger in Dec-POMDPs (NEXP-complete vs P-complete). And 2) commonly in Dec-POMDPs, there are many possible optimal policies, where each of these only uses a small part of all possible states-action pairs during self-play. For learning an optimal policy, it is sufficient for the model to cover the state-action space associated with any \emph{one} of these equilibria. Therefore in multi-agent settings, it is usually unnecessary to learn a good model of the entire environment.

\begin{figure*}[t!]

    \begin{subfigure}[t]{0.38\textwidth}
    \centering
    \includegraphics[width=\textwidth]{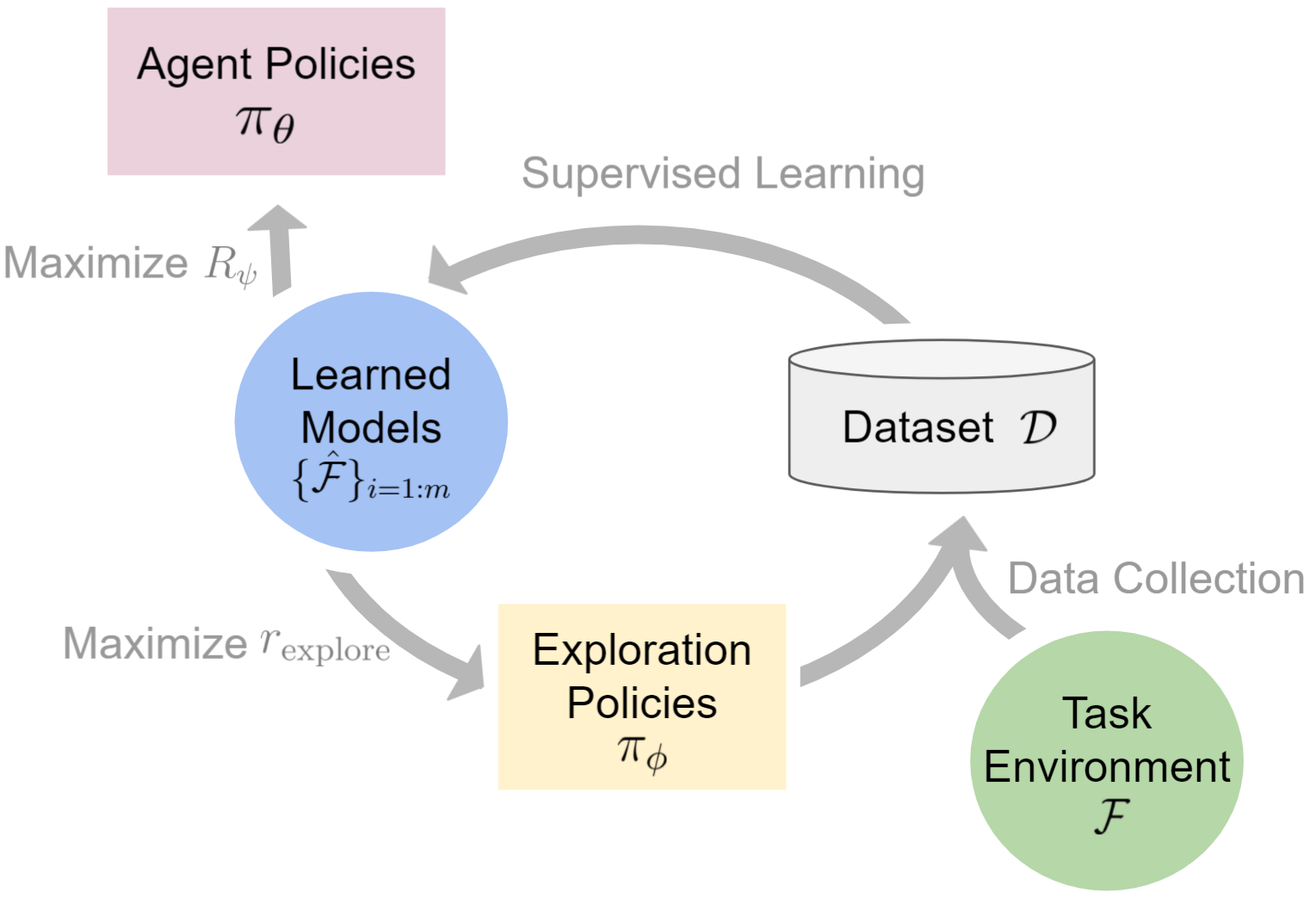}
    \caption{MARCO algorithm.}
    \end{subfigure}
    \begin{subfigure}[t]{0.28\textwidth}
    \centering
    \includegraphics[width=\textwidth]{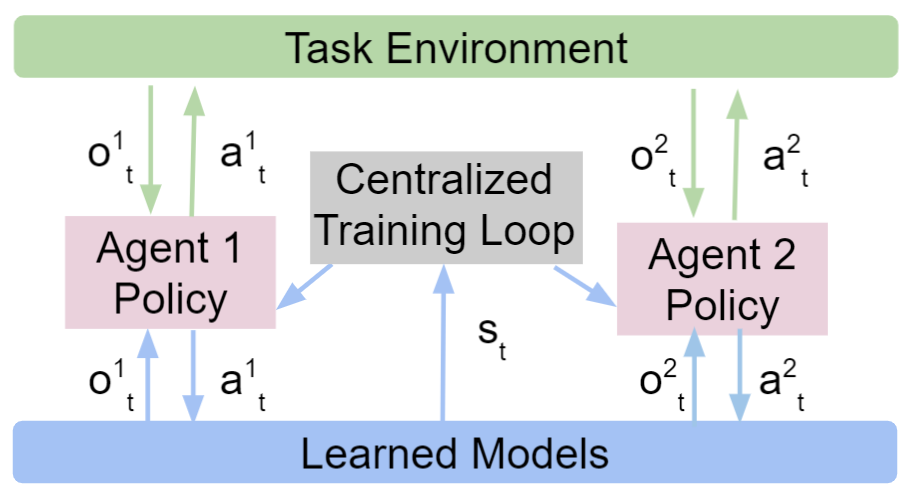}
    \caption{Agent policies (Observation as input).}
    \end{subfigure}
    \begin{subfigure}[t]{0.28\textwidth}
    \centering
    \includegraphics[width=\textwidth]{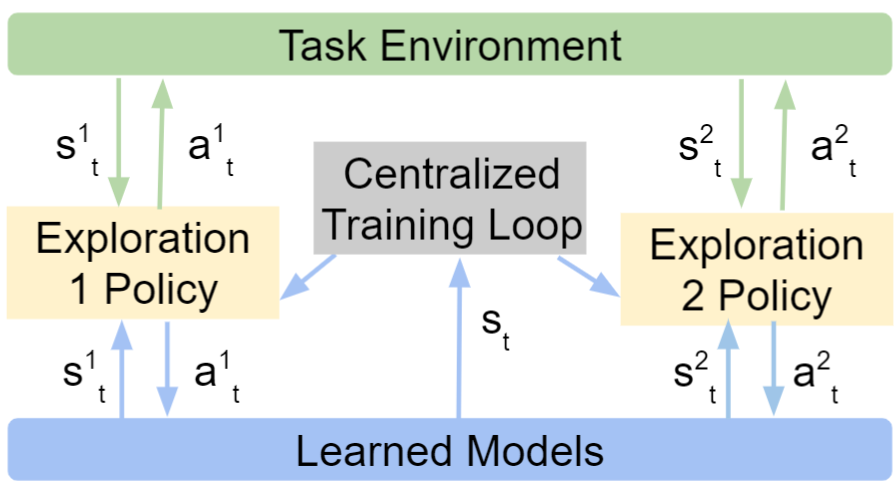}
    \caption{Exploration policies (State as input).}
    \end{subfigure}

    \caption{Overview of the MARCO algorithm. Blue arrows indicate information flow  during training, and green arrows indicate during test time. \textbf{(a)} MARCO alternates between policy optimization and model learning. During model learning, MARCO learns a centralized model that approximates the task environment. During policy optimization, MARCO updates the agents' policies using any model-free MARL algorithm of choice \textit{within the learned model}. MARCO also learns a centralized exploration policy, which is used to collect data in the task environment for model learning. \textbf{(b)} Information flow between the agents' policies and the environment. The agents' policies are decentralized (i.e. they only take individual observations as input). \textbf{(c)} Information flow between the exploration policies and the environment. The exploration policies are centralized (i.e. they take the state as input). Instead of a single exploration policy across the joint-action space, which grows exponentially with the number of agents, MARCO uses one exploration policy per agent.}
    \label{fig:MARCO}
    \end{figure*}

\noindent \textbf{Summary of  contributions}. This work makes four key contributions: 
\begin{enumerate}[wide, labelindent=0pt]
    \item First, we propose MARCO, a model-based learning algorithm for Dec-POMDPs. MARCO leverages centralized training to learn a model of the environment that generalizes across different policies. Within this model, we optimize the agents' policies using standard model-free methods, without using additional samples from the true environment.
    \item In most existing Dyna-style RL algorithms~\cite{sutton1991dyna, sutton1990integrated}, data for model learning is collected using the agent's current policy. This is inefficient when data in the same state-action space is re-collected. To further improve sample complexity, MARCO uses a centralized exploration policy. This policy specifically collects data in parts of the state-action space with high model uncertainty and is trained \emph{inside the model} to avoid consuming additional environment samples.
    \item To analyze the theoretical sample efficiency of model-based RL methods in Dec-POMDPs, we adapt R-MAX \cite{brafman2002r}, a model-based dynamic programming algorithm for MDPs, to tabular Dec-POMDPs. 
    Like MARCO, our adapted R-MAX also learns centralized models and performs exploration through ``optimism in the face of uncertainty'' \citep{brafman2002r}. 
    We prove that this adapted R-MAX algorithm achieves a sample complexity polynomial in the size of the state and the joint-action space.
    \item And finally, we conduct empirical studies comparing MARCO with model-free state-of-the-art multi-agent RL (MARL) algorithms in three cooperative communication tasks, where MARCO improves sample efficiency by up to 20x.
\end{enumerate}

\section{Related Work}

\paragraph{Single Agent Model-Based RL}
In RL problems, we generally do not assume prior knowledge of the environment. Model-free methods learn policies from interacting with the environment. In contrast,  model-based RL (MBRL) methods first learn a model of the environment and use the model in turn for decision making. MBRL is well explored in the context of single-agent RL, and has recently shown promising results ~\cite{ha2018world, hafner2019dream,schrittwieser2020mastering} across a variety of tasks~\cite{todorov2012mujoco, bellemare2013arcade}.

One problem of MBRL is that learning a perfect model is rarely possible, especially in environments with complex dynamics. In these settings, overfitting to model errors often hurt the test time performance. A popular method for addressing this problem is  to learn an ensemble of models ~\cite{kurutach2018model} and selecting one model randomly for each rollout step. 
Furthermore, the variance across the different models is a proxy for model uncertainty, which is used by e.g. \citet{kalweit2017uncertainty, buckman2018sample, yu2020mopo}. During training in areas of high model-uncertainty these methods either penalize the agent or fall back to the real environment. A different approach is to actively explore in state-action space with high model uncertainty to learn better models~\cite{ball2020ready, sekar2020planning}. Our work uses the latter approach of active exploration, and closely aligns with~\cite{sekar2020planning}, where we also explicitly learn an exploration policy within the learned model to perform data collection.

\paragraph{Model-free MARL for Dec-POMDPs}
Most deep RL work on learning in Dec-POMDPs uses model-free approaches. These methods can be roughly divided into two classes, value-based methods, e.g.~\cite{rashid2018qmix, tampuu2017multiagent,sunehag2017value}, which build on DQN~\cite{mnih2015human} and DRQN~\cite{hausknecht2015deep}, and actor-critic methods, such as \cite{lowe2017multi,foerster2018counterfactual}. These methods show good results in many tasks, but the number of samples required often goes into the millions or billions, even for environments with discrete or semantically abstracted state-action spaces~\cite{samvelyan2019starcraft,bard2020hanabi}. Orthogonal to our approach, few works have been proposed to address the sample complexity problem of MARL algorithms using off-policy learning methods ~\cite{vasilev2021semi, jeon2020scalable}.
\vspace{-0.75em}
\paragraph{Model-based work in MARL}
A popular branch of work in the multi-agent setting studies opponent modelling. Instead of learning the dynamics of the environment, agents infer the opponents' policy from observing their behaviour to help decision making \cite{brown1951iterative,foerster2017learning, mealing2015opponent}. Along this line of work, \citeauthor{wang2020model} trains an explicit model for each agent that predicts the goal conditioned motion of all agents in the partially observable environment.

There is little research in MARL that learns a model of the environment dynamics, as we do in our work (i.e. predicting the successor state from a state-action pair). \citeauthor{zhang2020model} theoretically analyse the sample complexity of model-based two-player zero-sum discounted Markov games, but do no present empirical studies. \citeauthor{krupnik2020multi} propose a multi-step generative model for two-player games, which does not predict the successor state, but the sequence of future joint-observation and joint-actions of all agents. Concurrent to our work, MAMBPO~\cite{willemsen2021mambpo}, most closely aligns with ours. Here, the authors learn a model of the environment, within which they perform policy optimization. While this work is concurrent to ours, there are also two key differences:  1) it does not learn a centralized exploration policy, and 2) MAMBPO uses the joint-observation and joint-action at the \textit{current} timestep to predict the next joint-observation and reward. In contrast to our \textit{fully centralized} model, which conditions on the central state, their model is not Markovian (see Figure \ref{fig:pgm}). We illustrate this by a simple example: Suppose other agents in the environment can flip a light switch, but lights actually only turn on after a delay of 10 timesteps, which is reflected by a count-down value in the central state (not observed by any of the agents). The joint-observation and joint-action alone at the current timestep is insufficient for predicting the next joint-observation. In this example, the history of at least 10 past joint-observation and joint-action is required. 

\section{Background}
\subsection{Dec-POMDPs}
\label{sec:setting}
We consider a fully cooperative, partially observable task that is formalized as a decentralized partially observable Markov Decision Process (Dec-POMDP)~\cite{oliehoek2016concise} $\mathcal{F} = \langle S, A, P, R, Z, O, N, \gamma, d_0 \rangle$.
$s \in S$ describes the central state of the environment, and $d_0$ is the initial state distribution. At each timestep, each agent $j \in J \equiv\{1, \ldots, N\}$ draws individual observations $o^j \in Z$ according to the observation function $O(s, \mathbf{a}): S \times \mathbf{A} \rightarrow \mathbf{Z}$. Each agent then chooses an action $a^{j} \in A,$ forming a joint-action $\mathbf{a} \in \mathbf{A} \equiv A^{N}$ \footnote{Bold notation indicates joint quantity over all agents.}. This causes a transition in the environment according to the state transition function $P(s' \mid s, \mathbf{a}): S \times \mathbf{A} \times S \rightarrow [0,1]$. All agents share the same reward function $R(s, \mathbf{a}): S \times \mathbf{A} \rightarrow \mathbb{R}$ and $\gamma \in[0,1)$ is a discount factor.  

 Each agent has an action-observation history (AOH) $\tau^{j} \in T \equiv(Z \times A)^*,$ on which it conditions a stochastic policy $\pi^{j}\left(a^{j} \mid \tau^{j}\right): T \times$ $A \rightarrow[0,1] .$ The joint-policy $\mathbf{\pi}$ induces a joint action-value function: $Q^{\pi}\left(s_{t}, \mathbf{a}_{t}\right)=\mathbb{E}_{s_{t+1: \infty}, \mathbf{a}_{t+1: \infty}}\left[R_{t} \mid s_{t}, \mathbf{a}_{t}\right],$ where $R_{t}=\sum_{i=0}^{\infty} \gamma^{i} r_{t+i}$ is the discounted return. 
 

\subsection{Dyna Style Model-Based RL}
RL algorithms fall under two classes: model-free methods, where we directly learn value functions and/or policies by interacting with the environment, and model-based methods, where we use interactions with the environment to learn a model of it, which is then used for decision making. Dyna-style algorithms~\cite{sutton1991dyna, sutton1990integrated} are a family of model-based algorithms for single-agent RL where training alternates between two steps: model learning and policy optimization. During model learning data is collected from the environment using the current policy and is used to learn the transition function. 
During policy optimization the policy is improved using a model-free RL algorithm of choice from data generated by the learned model. 

\subsection{Model-Free Multi-Agent Approaches}
Most MARL methods for approximately solving Dec-POMDPs fall in the category of model-free methods. Many use the centralized training for decentralized execution (CTDE) framework \cite{foerster2018counterfactual,kraemer2016multi, oliehoek2008optimal}, i.e., the learning algorithm has access to all global information, such as the joint-actions and the central state, but, at test time, each agent's learned policy conditions only on its own AOH $\tau^{j}$.

A popular branch of multi-agent methods for partially observable, fully cooperative settings is based on Independent Q-Learning (IQL)~\cite{tan1993multi,tampuu2017multiagent}. IQL treats the Dec-POMDP problem as simultaneous single-agent problems. Each agent learns its own Q-value that conditions only on the agent's own observation and action history, treating other agents as a part of the environment. MAPPO~\citep{chao2021surprising}, another independent learning algorithm, extends PPO~\citep{schulman2017proximal} to Dec-POMDPs. 
The advantage of independent learning is that it factorizes the exponentially large joint-action space. However, due to the nonstationarity in the environment induced by the learning of other agents, convergence is no longer guaranteed. Works like VDN and QMIX~\cite{sunehag2017value,rashid2018qmix} partially address this issue by learning joint Q-values. The former uses the sum of value functions of individual agents as the joint Q-values, while the latter learns a function parameterized by a neural network to map from individual Q-values to joint Q-values using the central state.

\begin{figure}[t!]
    \centering
    \includegraphics[width=0.77\columnwidth]{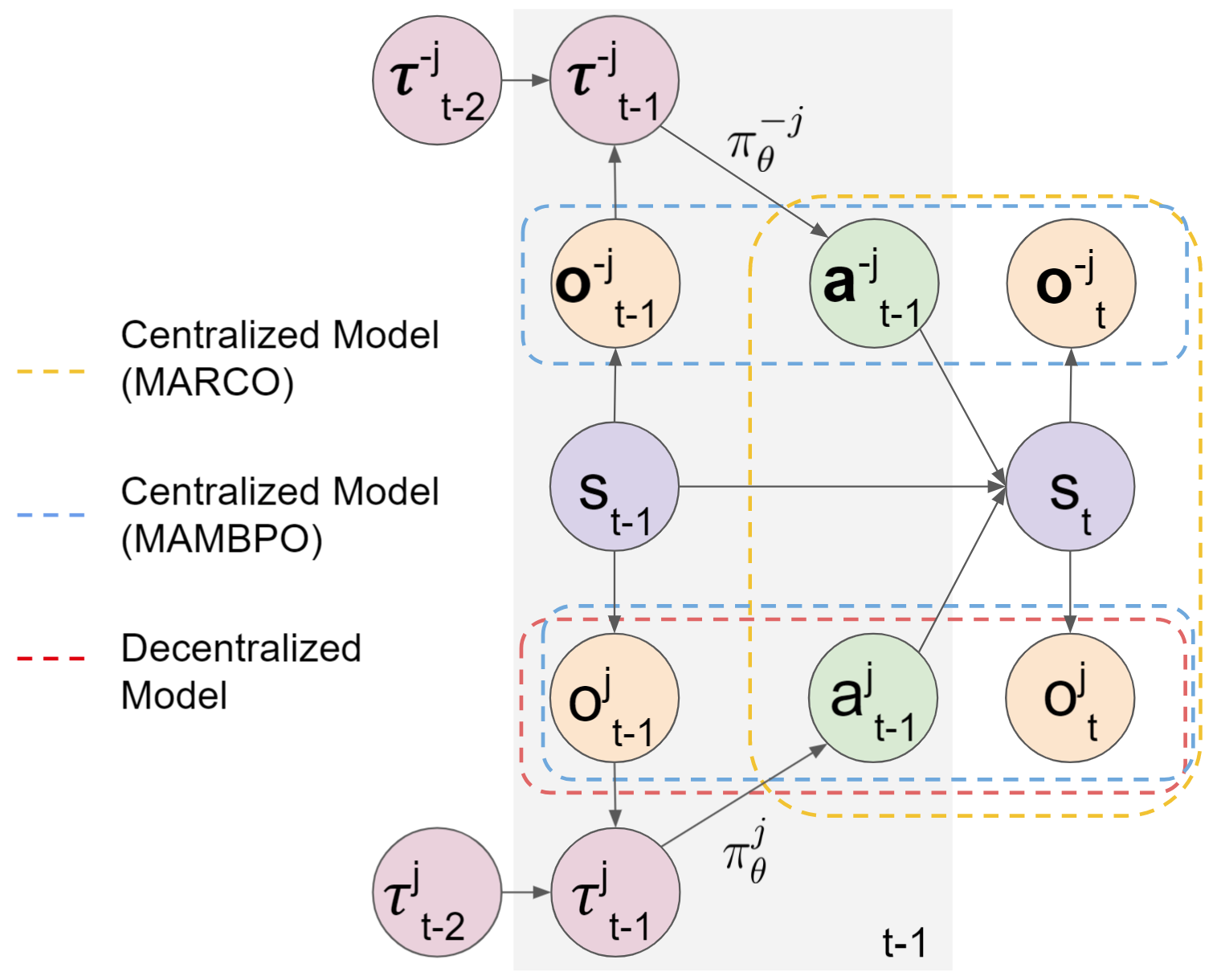}
    \caption{Probabilistic graphical models of the different possible learned Dec-POMDPs observation models. We use the notation $-j$ to denote the set of agents $\{1,..,n\} \setminus {j}$. The \textbf{yellow} box shows the centralized observation model $O_\psi (\mathbf{o_{t}}|s_{t}, \mathbf{a_{t-1}})$ learned by MARCO. The MARCO model is Markovian because the central state is used. It is also stationary since the environment dynamics are assumed to be stationary. The \textbf{blue} box shows the partially centralized model $O(\mathbf{o_{t}}|\mathbf{o_{t-1}},\mathbf{a_{t-1}})$ learned by MAMBPO \cite{willemsen2021mambpo}. The model is non-Markovian because using the current timestep's joint-observation instead of the central state is insufficient to make the prediction. The \textbf{red} box shows the decentralized model $O^j(o^j_{t}|o^j_{t-1}, a^j_{t-1})$, which is also non-Markovian because the central state is not used. It is also non-stationary because the model depends on the policies of other agents $\mathbf{\pi^{-j}_\theta}$, which are updated throughout training.}
    \label{fig:pgm}
\end{figure}

\section{Methods}
Optimally solving, or even finding an $\epsilon$-approximate solution for, Dec-POMDPs is NEXP-complete~\cite{bernstein2002complexity,rabinovich2003complexity}, which is significantly harder than solving MDPs with a complexity of P-complete~\cite{papadimitriou1987complexity}.  
This provides a strong motivation for using a model-based approach in Dec-POMDPs, as the number of samples required for learning a centralized model is polynomial in the state-action space, like in single-agent-RL~\cite{strehl2009reinforcement}. 

To learn a policy in Dec-POMDPs, MARCO (see Figure \ref{fig:MARCO}) adapts Dyna-style model-based RL to the multi-agent setting, as shown in Algorithm \ref{alg:MARCO}\footnote{For ease of notation, the input to all components $\hat{f} \in \hat{\mathcal{F}}$ in Algorithm 1 is written as $s_t, \mathbf{a_t}$. Actual inputs of $\hat{f}$ vary based on which model component $\hat{f}$ is.}, which alternates between learning an approximation of the Dec-POMDP, $\hat{\mathcal{F}}$, and optimizing the policy within $\hat{\mathcal{F}}$. We refer to  $\hat{\mathcal{F}}$ as \textit{the model} for the remainder of the paper. The two key contributions of MARCO are 1) learning a stationary model via centralized training, and 2) actively collecting data using a separate centralized \emph{exploration policy} trained inside the model not requiring additional environment samples.

\subsection{Model-Based MARL with Centralized Models} 
 The model $\mathcal{\hat{F}}$ is composed of the following components, each of which is a parameterized, learned approximation of the original Dec-POMDP $\mathcal{F}$:
 \vspace{-5pt}
\begin{align*}
  \text{Reward model}: \qquad & R_\psi(r_t|s_t, \mathbf{a_{t-1}}, s_{t-1}) \\
  \text{Dynamics model}: \qquad & T_\psi(s_{t+1}|s_t, \mathbf{a_t}) \\
  \text{Observation model}: \qquad & O_\psi(\mathbf{o_t}|s_t, \mathbf{a_{t-1}}) \\ 
  \text{Termination model}: \qquad & {P_\mathrm{term}}_\psi(\mathrm{termination}|s_t, \mathbf{a_t})
\end{align*}
In the single-agent, fully observable MBRL setting, the agent learns a dynamics model, and sometimes a reward and termination model. In MARCO, we also learn an observation model. 

MARCO takes advantage of CTDE by letting all of our models condition on the central state as well as the joint-action. The importance of centralized model learning can be illustrated as follows: For example (see Figure \ref{fig:pgm}), if the observation model is learned in a decentralized fashion (i.e. $O(o_t^j|o_{t-1}^j, a_{t-1}^j)$), then the actions of other agents, even when unobserved, can change the transition function of agent $j$ (i.e. if another agent turns on a light, agent $j$ will observe that the light transitioned from ``off'' to ``on''). As agents' policies are changing throughout training, the observation model thus is non-stationary and would have to be re-learned at various stages during training. In contrast, MARCO learns a fully centralized model which approximates the stationary \emph{ground truth} Dec-POMDP. Crucially, while in Dec-POMDPs agents have to explore over policies, MARCO allows agents to learn a single stationary model that is simultaneously accurate for \emph{all different} policies being explored.

Each component of the model is parametrized by a separate neural network and is trained using supervised learning through maximizing the likelihood of the collected data. Similar to~\citeauthor{kurutach2018model}, we train an ensemble of models to prevent the policy from overfitting to and exploiting model errors. When generating rollout data using the model, we randomly sample which one of the ensemble models to generate from. 

\begin{algorithm}[t!]
\caption{MARCO: Multi-Agent RL with Centralized Models and Exploration}
\label{alg:MARCO}
\begin{algorithmic}[1]
  \State \textbf{Input}: Number of ensemble models $m$, and uncertainty hyper-parameter $\lambda$.
  \State Initialize action-value functions $Q_\theta$, $Q_\phi$.
  \State Initialize $m$ ensemble environment models $\hat{\mathcal{F}}_i = \{R_{i,\psi},T_{i,\psi}, O_{i,\psi},  
  P_\mathrm{term,i_\psi}\}$ for $i=1,.., m$.
  \State Initialize a dataset $D$ with samples collected using a random policy from the real environment $\mathcal{F}$.
  \Repeat
  \State Train model $\hat{\mathcal{F}}$ using dataset $D$.
\State Update $Q_{\theta}$ using model-free algorithm of choice within a randomly chosen model $\hat{\mathcal{F}}_i$.
\State Update $Q_\phi$ using model-free algorithm of choice in a centralized fashion within a randomly chosen model $\hat{\mathcal{F}}_i$ as follows:
\begin{equation}
\label{eq:central_exploration}
\begin{aligned}
    \tilde{r}(s_t, \mathbf{a_t}) &= \sum_{\hat{f}\in \hat{\mathcal{F}}} \sum_{k=1}^{d(\hat{f})} Var(\{\hat{f}_{k,i, \psi}(s_t, \mathbf{a_t})\}_{i=1,..,m})\\
    r_{\mathrm{explore}} &= R_\psi(s_t, \mathbf{a_t}, s_{t-1}) + \lambda \tilde{r}(s_{t}, \mathbf{a}_{t}) \\
    y &= r_{\mathrm{explore}} +\gamma \max_{\mathbf{a}^{\prime}}\sum_{j=1}^{n} Q_{\text{target}}\left( s_{t+1}, a^{\prime} ,j; \phi^{-}\right)\\
  \mathcal{L}(\phi) &= \sum_{i=1}^{b}\left[\left(y_{i}-\sum_{j=1}^{n}Q(s, \mathbf{a}, j; \phi)\right)^{2}\right],
\end{aligned}
\end{equation}  
where $d(\hat{f})$ is the dimensionality of the model component $\hat{f}$.

\State Collect samples from environment $\mathcal{F}$ using centralized exploration policy $\pi_\phi$ and add them to $D$.
\Until{Maximum environment samples is used.}

\end{algorithmic}
\end{algorithm}

\subsection{Model-Based MARL with Centralized Exploration Policy}
MARCO collects data for model learning from a separate exploration policy $\pi_\phi$. Ideally, we want to collect data in regions of the state-action space with high model uncertainty. To quantify this \textit{epistemic} uncertainty, we use the variance of the models in the ensemble, which we denote as $\tilde{r}$ in equation \ref{eq:central_exploration}. To prevent the exploration policy from wandering off to regions irrelevant to the search space of the policies, the exploration policy should also optimize for the original objective. Hence, we set the reward of the exploration policy as the linear combination of $\tilde{r}$ and the reward generated by the reward model $R_\psi$. The hyper-parameter $\lambda$ controls the trade-off between exploration and exploitation. The exploration policy is learned entirely in the model, without using additional samples from the ground truth environment. To train the exploration policy, we use VDN (or QMIX), and again fully exploit CTDE by using the central state inside the model.
By conditioning on the central state, the exploration policy is able to more quickly return to the frontier, where high model uncertainty starts to occur. This avoids repeating data collection in regions of the already known state-action space, allowing more sample efficient model learning.  

Note that although MARCO's exploration policies uses centralized information, each agent's exploration policy outputs only their respective action. This is opposed to learning a single exploration policy that outputs the joint-action, which becomes intractable due to the exponential joint-action space.

\section{Sample Complexity in Tabular Dec-POMDPs}
We investigate the theoretical sample complexity of model-based Dec-POMDP methods. To do so, we make four additional assumptions that are not required for MARCO a) discrete and finite state, observation and action space b) at each timestep, the reward is bounded $0 \leq r(s, \mathbf{a}) \leq 1$, c) finite horizon, and d) deterministic observation function. Under these assumptions, we show that an idealized model-based method achieves a sample complexity polynomial in the size of the state and joint-action space. 

 We modify R-MAX \citep{brafman2002r}, an MBRL algorithm for MDPs, to the Dec-POMDP setting (see Algorithm \ref{alg:rmax}). We refer to our modified algorithm as the Adapted R-MAX for the remaining of the paper. Like MARCO, the Adapted R-MAX aims to learn a near-optimal decentralized joint-policy for a given Dec-POMDP $D$. Our adaptation of R-MAX also takes full advantage of CTDE by learning centralized models $\hat{P}(s,\mathbf{a})$, $\hat{R}(s,\mathbf{a})$, and $\hat{O}(s)$ using empirical estimates. Using the centralized models, the Adapted R-MAX constructs an approximate \textit{K-known Dec-POMDP} $\hat{D}_K$ (see Definition \ref{def:known}), where $K$ is the set of state-action pairs that has been visited at least $m$ times. Within the model $\hat{D}_K$, we then evaluate all possible joint-policies $\mathbf{\pi} \in \mathbf{\Pi}$ and choose the best one. Both MARCO and our adaptation of R-MAX encourage exploration in parts of the state-action space with high model uncertainty. The former performs exploration through a separate centralized exploration policy, while the latter performs exploration through optimistically setting the reward function of under-visited state-action pairs (i.e. those not in $K$).

\begin{definition}[K-Known Dec-POMDP]
\label{def:known}
    $D_K$ is the expected version of $\hat{D}_K$ where:
    
$$
\begin{aligned}
    P_{K}\left(s^{\prime} \mid s, \mathbf{a}\right) & =\left\{\begin{array}{lll} P\left(s^{\prime} \mid s, \mathbf{a}\right) & \text { if } (s, \mathbf{a}) \in K \\ \mathbbm{1}\left[s^{\prime}=s\right] & \text { otherwise} & \end{array}\right.\\
    \hat{P}_{K}\left(s^{\prime} \mid s, \mathbf{a}\right) &= \begin{cases} \frac{n\left(s, \mathbf{a}, s^{\prime}\right)} {n(s, \mathbf{a})}, & \text { if }(s, \mathbf{a}) \in K \\ \mathbbm{1}\left[s^{\prime}=s\right], & \text { otherwise }\end{cases} \\
     R_{K}\left(s, \mathbf{a}\right)&=\left\{\begin{array}{lll} R\left(s, \mathbf{a}\right) & \text { if } (s, \mathbf{a}) \in K \\ R_\text{max} & \text { otherwise} & \end{array}\right. \\
     \hat{R}_{K}(s, \mathbf{a})&= \begin{cases}\frac{\sum_{i}^{n(s,\mathbf{a})}r(s, \mathbf{a})}{n(s,\mathbf{a})}, & \text { if }(s, \mathbf{a}) \in K \\ R_{\max }, & \text { otherwise }\end{cases} \\
    O_{K}\left(s\right)&=\left\{\begin{array}{lll} O\left(s\right) & \text { if } (s, \cdot) \in K \\ \text{random observation} & \text { otherwise} & \end{array}\right. \\
     \hat{O}_{K}\left(s\right) &=\left\{\begin{array}{lll} \mathbf{o} \text{ where } n\left(s,  \mathbf{o}\right) > 0 & \text { if } (s, \cdot) \in K \\ \text{random observation} & \text { otherwise} & \end{array}\right.
\end{aligned}
$$

\end{definition}

\begin{algorithm}[h!]
\caption{Adapted R-MAX for Dec-POMDPs}
\label{alg:rmax}
\begin{algorithmic}[1]
    \State \textbf{Input}: $\gamma, m$.
    
    \For {all $(s, \mathbf{o})$}
        \State $n(s, \mathbf{o}) \leftarrow 0$
    \EndFor

    \For {all $(s, \mathbf{a})$ }
        \State $r(s, \mathbf{a}) \leftarrow 0$
        \State $n(s, \mathbf{a}) \leftarrow 0$
        \For {all $s^{\prime} \in S$}
            \State $n(s, \mathbf{a}, s^{\prime}) \leftarrow 0$
        \EndFor
    \EndFor
    \For {$t=1,2,3,\cdots$}
        \State Let $s, \mathbf{o}$ denote the state and observation at time $t$ respectively.
        \State Choose action $\mathbf{a}$ according to $\mathbf{\pi}^*_{\hat{D}_K}$
        \State Let $r$ be the immediate reward and $s^{\prime}$ the next state after executing action $\mathbf{a}$ from state $s$.

        \If{$n(s, \mathbf{a})<m$}
            \State $n(s,\mathbf{o}) \leftarrow 1$ // Record observation
            \State $n(s, \mathbf{a}) \leftarrow n(s, \mathbf{a})+1$
            \State $r(s, \mathbf{a}) \leftarrow r(s, \mathbf{a})+r / /$ Record immediate reward
            \State $n(s, \mathbf{a}, s^{\prime}) \leftarrow n(s, \mathbf{a}, s^{\prime})+1 / /$ Record immediate next-state
        
        \EndIf
        \If {$n(s, \mathbf{a})=m$}
                \For {all $\mathbf{\pi} \in \mathbf{\Pi}$}:
                    \State Obtain $J_{\hat{D}_K}(\mathbf{\pi})$ using Monte Carlo rollouts in $\hat{D}_K$.
                \EndFor
                \State $\mathbf{\pi}^*_{\hat{D}_K} \leftarrow \arg \max_{\mathbf{\pi}} J_{\hat{D}_K}(\mathbf{\pi})$
        \EndIf
    \EndFor
\end{algorithmic}
\end{algorithm}

To study the sample complexity of the Adapted R-MAX, we define the value of a joint-policy as follows:
\begin{definition}
    Given a decentralized joint-policy $\mathbf{\pi}$, we estimate its value $J(\mathbf{\pi}) $ in a Dec-POMDP $D$, defined as the expected reward obtained by following the joint-policy in $D$,
\begin{equation}
         J_D(\mathbf{\pi}) = \E_{s\sim d_0} \left[V^\pi(s)\right].
\end{equation}
\end{definition}
\begin{theorem}
\label{thm:pac}
Suppose that $0 \leq \varepsilon<1$ and $0 \leq \delta<1$ are two real numbers and $D$ is any Dec-POMDP. There exists inputs $m=m\left(\frac{1}{\varepsilon}, \frac{1}{\delta}\right)$ and $\varepsilon$, satisfying $m\left(\frac{1}{\varepsilon}, \frac{1}{\delta}\right)=O\left(\frac{(S+\ln (S \mathbf{A} / \delta)) V_{\max }^{2}}{\varepsilon^{2}(1-\gamma)^{2}}\right)$ such that if the Adapted R-MAX algorithm is executed on $D$ with inputs $m$ and $\varepsilon$, then the following holds: Let $\mathbf{\pi}^*_{\hat{D}_K}$ denote the Adapted R-MAX 's policy. With probability at least $1-\delta$, $J_D(\mathbf{\pi}^*) - J_D(\mathbf{\pi}^*_{\hat{D}_K}) \leq 2\epsilon$ is true for all but
$$O\left(\frac{|S||\mathbf{A}|}{(1-\gamma)^2\epsilon^3} \left(|S| + \ln(\frac{|S||\mathbf{A}|}{\delta})\right)V_{\max}^3 \ln \frac{1}{\delta}\right)$$ episodes.
\end{theorem}

Theorem \ref{thm:pac} shows that Adapted R-MAX acts near-optimally on all but a polynomial number of steps. These results confirm the motivation of MARCO, i.e. that an idealized MBRL method for Dec-POMDPs can indeed have polynomial sample complexity.
The proof (see the appendix) is heavily based on results from \citeauthor{jiang2020notes} and \citeauthor{strehl2009reinforcement}, but for the first time extends them to the Dec-POMDP setting.

\section{Experiments}
\label{experiment_section}

\subsection{Environments}

\begin{figure}[t!]
    \centering
    \begin{subfigure}[t]{0.42\columnwidth}
        \centering
        \includegraphics[width=\columnwidth]{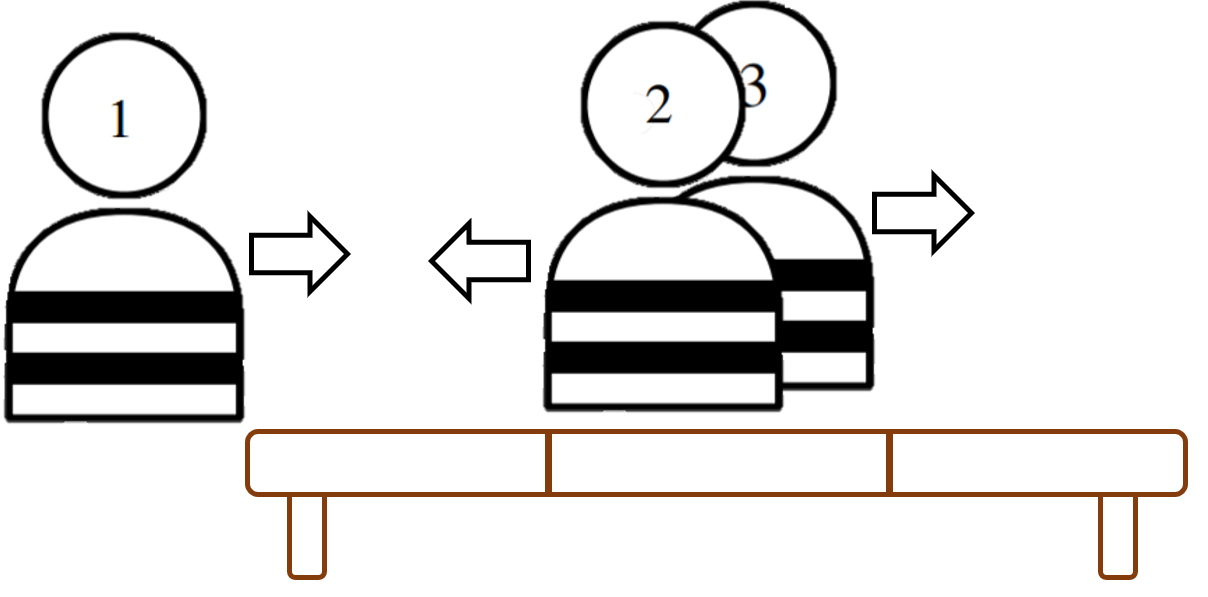}
        \caption{Bridge crossing phase.}
        \label{switch:a}
    \end{subfigure}
    \begin{subfigure}[t]{0.94\columnwidth}
        \centering
        \includegraphics[width=\columnwidth]{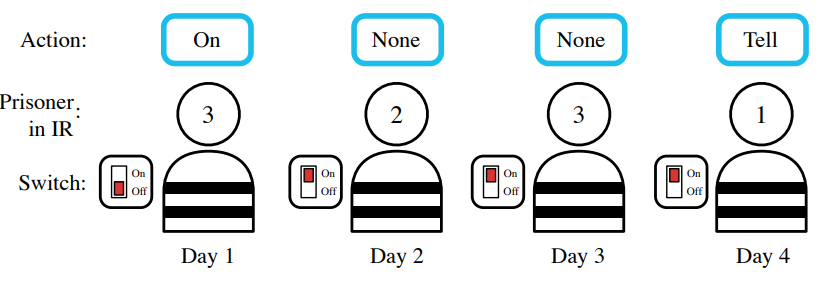}
        \caption{ Switch riddle playing phase.}
        \label{switch:b}
    \end{subfigure}
    \caption{\textit{Switch riddle with bridge}. \textbf{(a)} The game begins with the \textit{bridge crossing} phase. All agents start on the left side of the bridge. At each timestep, each agent chooses an action from $\{$``Left’’,  ``Right’’, and ``End episode’’$\}$. \textbf{(b)} Switch riddle \textit{playing phase}~\cite{foerster2016learning}. After all agents arrive at the right side of the bridge, every timestep one agent gets sent to the interrogation room where they see the switch and choose an action from $\{$``On’’, ``Off’’, ``Tell’’, ``None’’, ``Left’’, ``Right’’, and ``End episode’’$\}$.}
    \label{fig:envs}
    \vspace{-10pt}
\end{figure}

We evaluate the sample efficiency of MARCO against model-free MARL algorithms on three fully cooperative, partially observable communication tasks, the switch riddle~\cite{foerster2016learning}, a variant of the switch riddle, and the simple reference game from the multi-agent particle environment (MPE)~\cite{lowe2017multi}. We explicitly chose communication tasks because one agent's belief is directly affected by other agents' policies, resulting in the larger policy search spaces typical for Dec-POMDPs.

\paragraph{Switch without Bridge~\cite{foerster2016learning}}
 At each timestep $t$, a random agent $j \in \{1,2,3\}$ is sent into the interrogation room for one timestep. Each agent observes whether it is currently in the room, but only the current agent in the room observes whether the light switch in the room is ``On'' or ``Off''. If agent $j$ is in the interrogation room, then its actions are $a_{t}^{j} \in\{$ ``None’’, ``Tell’’, ``Turn on lights’’, ``Turn off lights’’$\}$; otherwise the only action is ``None’’. The episode ends when an agent chooses ``Tell’’ or when the maximum timestep, $T$, is reached. The reward $r_{t}$ is 0 unless an agent chooses ``Tell’’, in which case it is 1 if all agents have been to the interrogation room, and $-1$ otherwise. Finally, to keep the experiments computationally tractable we set the time horizon to $T=6$.

 \paragraph{Switch with Bridge} To make the first task more challenging, we modify it as follows (see Figure \ref{fig:envs}). All agents start on the left side of a bridge, and the switch riddle only starts once all agents have crossed a bridge of length $3$. If at timestep $t$ not all agents have crossed the bridge, each agent observes its position $o_t^{j} \in \{0,1,2,3\}$ on the bridge, and its actions are $a_{t}^{j} \in\{$  ``Left’’,  ``Right’’,  ``End episode’’$\}$. Selecting the action ``Left’’ and ``Right’’ increments the agent's position by -1 and +1 respectively. The episode ends and agents receive a reward of 0 if ``End episode’’ is chosen. When all agents have crossed the bridge (i.e. all agents are at position $3$ on the bridge), the switch riddle starts. Now agents proceeds like the above task, except that the agent currently in the room has access to additional actions $a_{t}^{j} \in\{$  ``None’’,  ``Tell’’,  ``Turn on lights’’,  ``Turn off lights’’,  ``Left’’,  ``Right’’,  ``End episode ’’$\}$. Selecting ``Left’’ or ``Right’’ is equivalent to selecting ``None’’, and selecting ``End episode’’  terminates the episode early with a return of $0$. We set the time horizon to $H=9$.
 
\paragraph{The Simple Reference Game} This task is a part of the multi-agent particle environment (MPE)~\cite{lowe2017multi}, and consists of two agents, that are placed in an environment with three landmarks of differing colors. At the beginning of every episode, each agent is assigned to a landmark of a particular color. The closer the agents are to their assigned landmark, the higher their reward. However, agents themselves don't observe their own assigned color, only the other agent's assigned color. At each timestep, each agent chooses two actions: A movement action $\in$ $\{$ ``Left’’,  ``Right’’, ``Up’’,  ``Down’’,  ``Do nothing’’$\}$, and one of the 10 possible messages that is sent to the other agent.

\subsection{Experiment Details \protect \footnote{See the appendix \cite{zhang2021centralized} for detailed experiment descriptions}}
\paragraph{Model Learning}
The dynamics model in the two switch tasks is an auto-regressive model implemented using GRUs \cite{chung2014empirical}. The remaining components of the model $\hat{\mathcal{F}}$ are implemented using fully connected neural networks. All model components are trained in supervised manner via maximum likelihood.

\paragraph{Dataset Collection}
The initial dataset is gathered with a random policy for all MARCO experiments. In the \textit{switch without bridge} task, 5k samples are collected from the environment after every 10k training steps in the model. No further data is collected beyond 10k samples.
For the \textit{switch with bridge} and the MPE tasks, 10k samples are collected from the environment after every 50k training steps in the model. No further data is collected beyond 50k samples. 

 \paragraph{Policy Optimization}
 The model-free baseline for each task is chosen by finding the most sample efficient algorithm between IQL, VDN, and QMIX. For each task, MARCO uses the same algorithm for policy optimization inside the model as the corresponding model-free baseline.

\begin{figure*}[t!]
    \centering
    \includegraphics[width=0.85\textwidth]{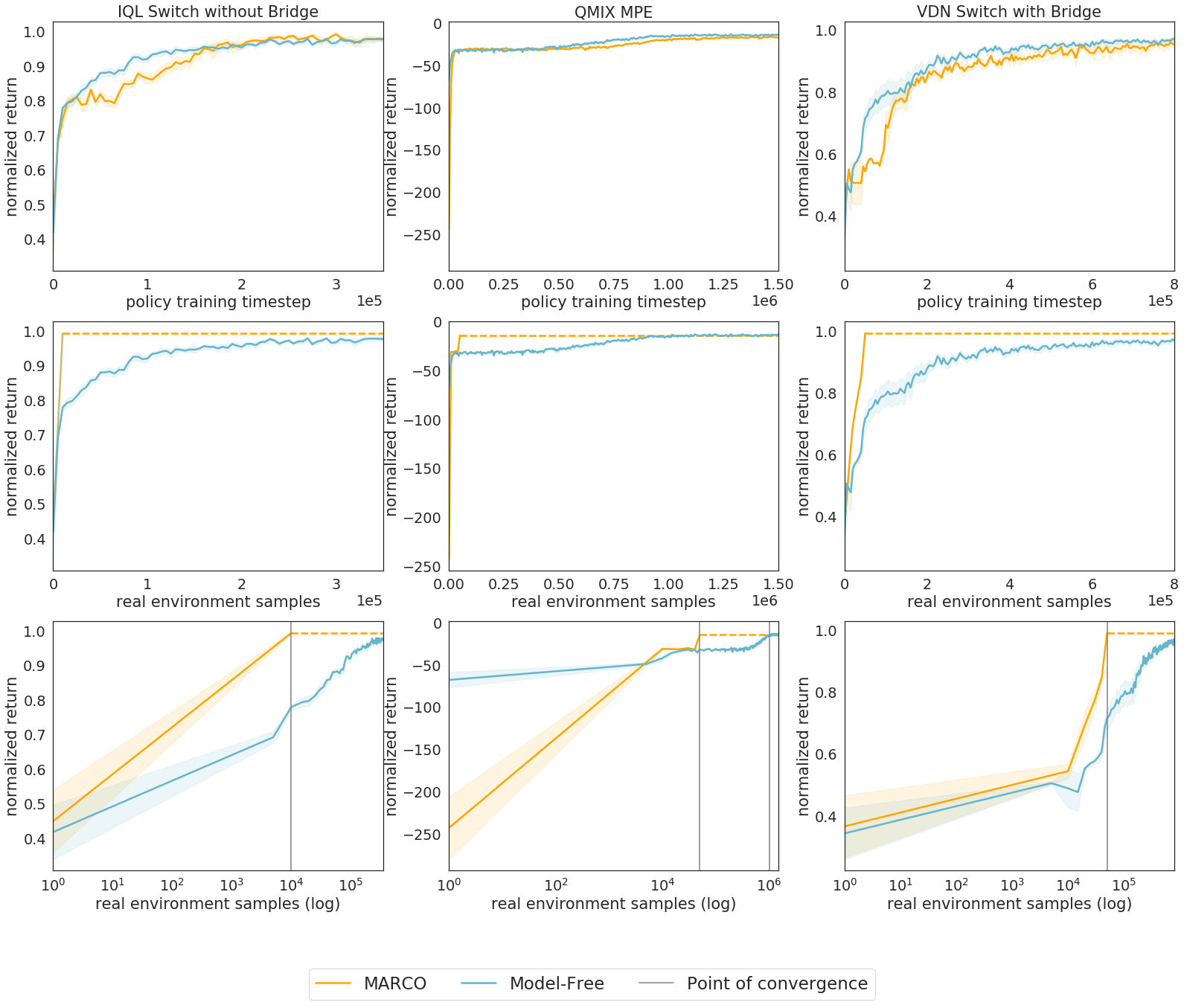}
    \vspace{-10pt}
    \caption{MARCO's test performance over 50 episodes with $\lambda=2.0$ matches model-free performance with much less environment samples.  The error bars reported are the standard error over 8 runs.
    \textbf{Left}: IQL in the \textit{switch without bridge} task. MARCO is 20x more sample efficient.
    \textbf{Middle}: QMIX in MPE. MARCO is 20x more sample efficient. 
    \textbf{Right}: VDN in the \textit{switch with bridge} task. The MARCO model is learned with 50k samples.  MARCO is 12x more sample efficient. }
    \label{fig:switch_result}
\end{figure*}

\subsection{Results}
\vspace{-0.4em}
The top row in Figure~\ref{fig:switch_result} displays results against \textit{policy training steps} to show MARCO matches model-free performance, while the middle row displays results against \textit{number of real environment interactions} to show sample efficiency of MARCO. The bottom row illustrates the performance against \textit{number of real environment interactions in log scale}, where we see a 1-2 order of magnitude improvement of sample-efficiency in MARCO over model-free methods.

The left column in Figure \ref{fig:switch_result} shows results for the \textit{switch with bridge} task. Model-free IQL learns the optimal policy in roughly 200k samples. MARCO learns the optimal policy with 10k samples, which is a sample efficiency increase of 20x. 

The middle column in Figure \ref{fig:switch_result} shows results for the MPE task. Model-free QMIX learns the optimal policy with roughly 1m samples, while MARCO learns it in 50k, an efficiency increase of 20x.

The right column in Figure \ref{fig:switch_result} shows results for the \textit{switch with bridge crossing}. MARCO learns the optimal policy with only 50k samples. In comparison, model-free VDN requires roughly 600k samples, about 12x more than MARCO.

\begin{figure}[h!]
    \centering
    \begin{subfigure}[t]{0.75\columnwidth}

    \includegraphics[width=\columnwidth]{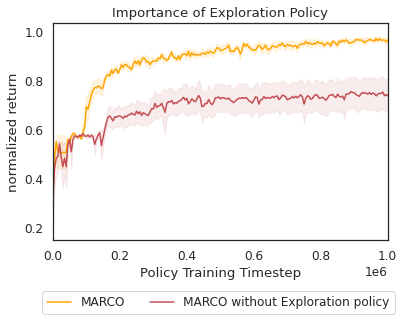}
    \caption{Importance of exploration policy.}
    \label{ab:a}
    \end{subfigure}
    
    \begin{subfigure}[t]{0.75\columnwidth}
    \includegraphics[width=\columnwidth]{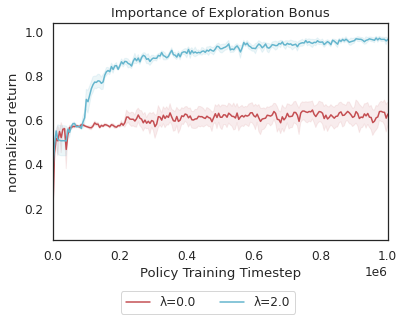}
    \caption{Importance of exploration bonus.}
    \label{ab:b}
    \end{subfigure}
    \caption{Ablation studies in the \textit{switch with bridge} task. \textbf{(a)} Without the exploration policy, MARCO agents do not learn the optimal policy.
    \textbf{(b)} Without the exploration bonus term $\tilde{r}$ (i.e. setting $\lambda$=0), MARCO's final policy is sub-optimal. 
    The error bars shown represent the standard error across 8 runs.}
    \label{fig:switch_lambda}
    \vspace{-15pt}
\end{figure}

\begin{figure}[t!]
    \centering
    \begin{subfigure}[t]{0.24\columnwidth}
    \includegraphics[width=\columnwidth]{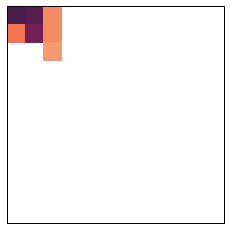}
    \caption{No exploration 50k.}
    \label{visitation:a}
    \end{subfigure}
    \begin{subfigure}[t]{0.24\columnwidth}
    \includegraphics[width=\columnwidth]{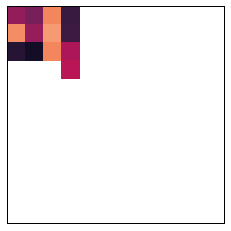}
    \caption{No exploration 100k.}
    \label{visitation:b}
    \end{subfigure}
        \begin{subfigure}[t]{0.24\columnwidth}
    \includegraphics[width=\columnwidth]{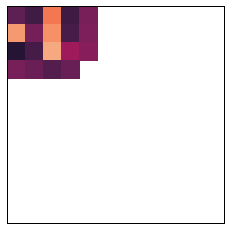}
    \caption{No exploration 150k.}
    \label{visitation:c}
    \end{subfigure}
    \begin{subfigure}[t]{0.24\columnwidth}
    \includegraphics[width=\columnwidth]{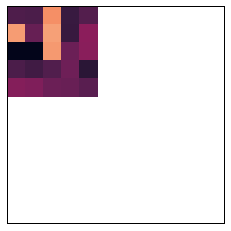}
    \caption{No exploration 200k.}
    \label{visitation:d}
    \end{subfigure}
    
    \begin{subfigure}[t]{0.24\columnwidth}
    \includegraphics[width=\columnwidth]{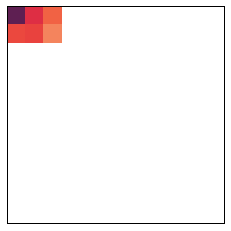}
    \caption{MARCO 50k.}
    \label{visitation:e}
    \end{subfigure}
    \begin{subfigure}[t]{0.24\columnwidth}
    \includegraphics[width=\columnwidth]{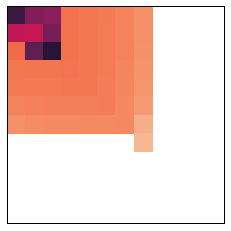}
    \caption{MARCO 100k.}
    \label{visitation:f}
    \end{subfigure}
        \begin{subfigure}[t]{0.24\columnwidth}
    \includegraphics[width=\columnwidth]{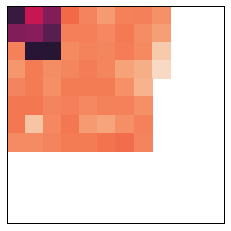}
    \caption{MARCO 150k.}
    \label{visitation:g}
    \end{subfigure}
    \begin{subfigure}[t]{0.24\columnwidth}
    \includegraphics[width=\columnwidth]{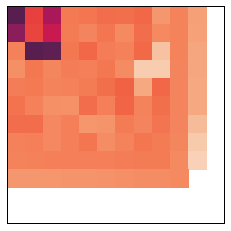}
    \caption{MARCO 200k.}
    \label{visitation:h}
    \end{subfigure}
    
    \includegraphics[width=0.7\columnwidth]{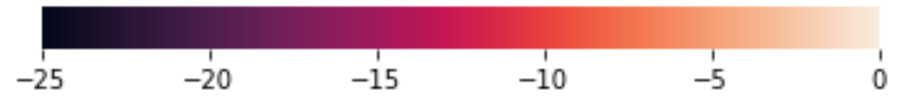}
    
    \caption{Log of the uncertainty term ($\log \tilde{r}$) of 5000 state-action pairs visited by the agents policies in the task environment either with (e-h) or without (a-d) the centralized exploration policy. The columns correspond to MARCO policies trained inside the model for 50k, 100k, 150k, and 200k timesteps, respectively. The $\log \tilde{r}$ term is calculated with the models that is available to the MARCO policy at the time of training. Darker color indicates less uncertainty, lighter color indicates more. No color (white) indicates an unvisited state-action pair. When the exploration policy is used, the agents cover much more of the state-action space starting at 100k than without the exploration policy.}
    \label{fig:visitation}
    \vspace{-15pt}
\end{figure}

\subsection{Ablation Studies} We evaluate the centralized exploration policy with two ablation studies in the \textit{switch with bridge} task.

\paragraph{Effect of using exploration policy} Figure \ref{ab:a} \footnote{The state-action space is displayed in 2D for ease of visualization} shows results comparing MARCO with an ablation without exploration policy. In the latter case, the model is learned with a dataset that is collected with the agents' current policy instead of a dedicated exploration policy. An $\epsilon$-greedy policy is used for data collection, where with a probability of $\epsilon=0.1$ the agent selects a random action from the available actions at the current step, and follows the data-collection policy otherwise.

MARCO without an exploration policy fails to solve this task because random exploration from the $\epsilon$-greedy data collection policy is insufficient to cover the relevant state-action space. If a random policy is used, at every timestep during the bridge crossing phase, there is a probability of $1-(\frac{2}{3})^3\approx0.3$ that at least one agent chooses ``End episode'' among the three available actions. Hence, the episode is very likely to terminate before the switch-riddle playing even begins, and so collecting data with a random exploration for the switch riddle task is difficult at the initial learning phase. In contrast, MARCO with centralized exploration policy overcomes this problem by actively exploring in the state-action space with high model-uncertainty, which is where agents cross the bridge and play the switch riddle. This is illustrated in Figure \ref{fig:visitation}, where we show that MARCO agents cover much more state-action space when using an exploration policy (Figure \ref{visitation:e}-\ref{visitation:h}) than otherwise (Figure \ref{visitation:a} - \ref{visitation:d}). The columns indicate the log of the uncertainty term of 5000 state-action pairs visited by MARCO’s agents policies in the task environment after training for 50k, 100k, 150k, and 200k timesteps in the model. The rollout is done within the model that is available to the agents' policies at the time. Darker color indicates less model uncertainty, lighter color indicates more. No color (white) indicates an unvisited state-action pair. 

\vspace{-0.75em}
\paragraph{Effect of using exploration bonus  $\mathbf{\tilde{r}}$} Figure \ref{ab:b} shows that when we do not use the exploration bonus term $\mathbf{\tilde{r}}$ by setting $\lambda=0$, MARCO no longer learns the optimal policy. This suggest that the exploration bonus term is essential. 

\vspace{-0.75em}

\section{Conclusion and Future Work}
\label{discussion_Section}
We presented MARCO, a model-based RL method adapted from the Dyna-style framework for sample-efficient learning in Dec-POMDPs. MARCO learns a centralized stationary model that in principle is entirely independent of the agents' policies. Within this model, policy optimization is performed using a model-free MARL algorithm of choice with no additional cost in environment samples (exploiting access to the \emph{simulated central state}). To further improve sample complexity, MARCO also learns a centralized exploration policy to collect samples in parts of the state-action space with high model uncertainty. In addition, to investigate the theoretical sample efficiency of model-based Dec-POMDP methods, we introduced the Adapted R-MAX algorithm for Dec-POMDPs, and showed that it achieves polynomial sample complexity. Finally, we showed on three cooperative communication tasks that MARCO matches the performance of state-of-the-art model-free MARL methods requiring significantly fewer samples.

We discuss the limitations of the work in two aspects. First of all, the centralized exploration policy may deviate from the agents' decentralized policies due to having access to additional information in the central state. This may cause data to be collected in parts of the state-action space that are inaccessible to the agents due to partial observability. 
Secondly, model-based methods commonly require more wall-clock time than model-free methods due to the additional model-learning step. 
However, by assumption, in our setting compute-time is cheap compared to environment interactions.

Despite these limitations, we hope that this work brings MARL one step closer to being applicable to real-world problems. In future work we will investigate how to choose a better centralized data exploration policy, 
as well as how to combine existing work in image-input single-agent MBRL to MARL settings to enable good performance even on complicated, image-based environments. 

\section{Acknowledgements}
Authors thank Wendelin Böhmer, Amir-massoud Farahmand, Keiran Paster, Claas Voelcker, and Stephen Zhao for insightful discussions and/or feedbacks on the drafts of the paper. QZ was supported by the Ontario Graduate Scholarship. AG and JF would like to acknowledge the CIFAR AI Chair award.

\clearpage
\newpage
\bibliographystyle{ACM-Reference-Format}
\bibliography{references}

\clearpage
\newpage
\onecolumn
\appendix
\section{Theoretical Results}
\subsection{Setting}

We operate in the same Dec-POMDP setup as described in Section \ref{sec:setting}, but with four additional assumptions:
\begin{itemize}
    \item Bounded rewards: $ \forall (s, \mathbf{a}), 0 \leq R(s, \mathbf{a}) \leq 1$. This allows us to introduce the notion of maximum reward and maximum value function as $R_{\max }=1$ and $V_{\max }=\frac{1}{1-\gamma}$ respectively.
    \item Finite horizon setting: All trajectories end in $H$ steps.
    \item Tabular setting: Discrete and finite number of states, actions and observations.
    \item Deterministic observation function. 
\end{itemize}

\subsection{Notation}
In addition to the setting described in Section \ref{sec:setting}, we introduce a few more notations for our sample complexity analysis:
\begin{itemize}
    \item $\hat{s} \in \hat{S}$ is the joint state-action history, i.e. $\hat{s}_2 = \{s_0, \mathbf{a}_0, s_1, \mathbf{a}_1, s_2\}$.
    \item $\hat{\tau}$ is the trajectory of joint state-action histories, i.e. $\hat{\tau}_2 = \{\hat{s}_0, \mathbf{a}_0, \hat{s}_1, \mathbf{a}_1, \hat{s}_2\}$
    \item $\mathbf{\pi}: \hat{S} \rightarrow (\mathbf{O} \times \mathbf{A})^* \rightarrow \mathbf{A}$ is a deterministic ``joint-policy’’ (see Figure \ref{fig:policy}). For simplicity, we absorb the deterministic observation function within the ``joint-policy’’ (i.e. the first mapping $\hat{S} \rightarrow (\mathbf{O} \times \mathbf{A})^*$ in $\mathbf{\pi}$ is the deterministic observation function, and the second mapping $(\mathbf{O} \times \mathbf{A})^* \rightarrow \mathbf{A}$ in $\mathbf{\pi}$ is the deterministic joint-policy we \textit{actually} learn). The second mapping is a set of decentralized policies that each maps from an individual AOH to an individual action. For the remainder of this section, when we talk about learning a near optimal joint-policy, we are referring to the actual set of decentralized joint-policies $(\mathbf{O} \times \mathbf{A})^* \rightarrow \mathbf{A}$. \textit{Hence, even though we denote $\mathbf{\pi}$ as a function that takes $\hat{s}$ as input, we are still learning decentralized policies}.
    \begin{figure*}[h!]
        \centering
        \includegraphics[width=0.6\textwidth]{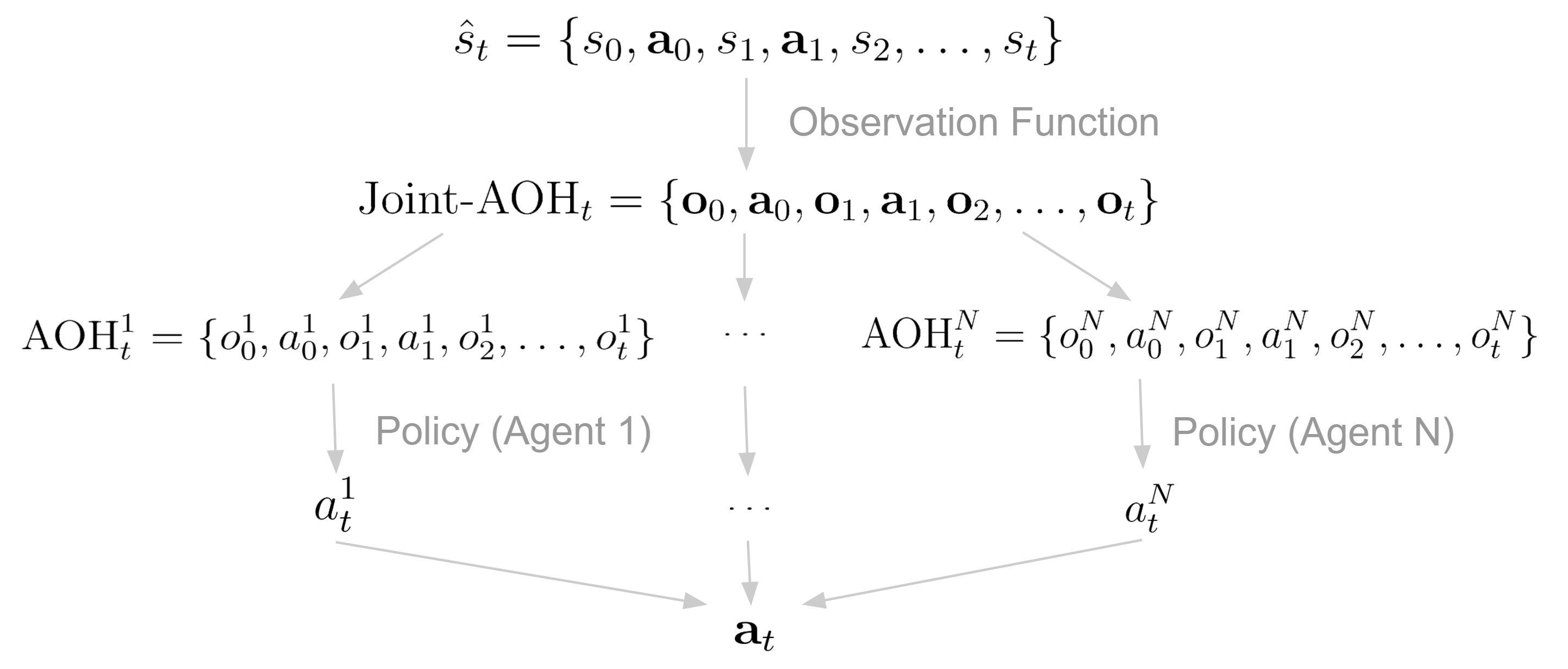}

        \caption{Illustration of the joint-policy $\mathbf{\pi}$ at timestep $t$ for a Dec-POMDP with $N$ agents.}
        \label{fig:policy}
    \end{figure*}

    \item ``$D$’’ denotes the Dec-POMDP, and ``$M$’’ denotes the centralized MDP.
    \item The relationship between reward functions in a Dec-POMDP and in its corresponding centralized MDP can be expressed as follows due to the Markov property in MDPs: $$R_D(\hat{s}, \mathbf{a}) = R_M(s, \mathbf{a}).$$
    \item Similarly, the relationship between the transition functions in a Dec-POMDP and in its corresponding centralized MDP can be expressed as follows due to the Markov property in MDPs: $$P_D(\hat{s}^\prime \mid \hat{s}, \mathbf{a}) = P_M(s^\prime \mid  s, \mathbf{a})\mathbbm{1}_\text{check history}.$$ The indicator function $\mathbbm{1}_\text{check history}$ checks whether $\hat{s}^\prime$'s previous joint state-action history is the same as $\hat{s}$, and whether the joint-action $\mathbf{a}$ is the same as $\hat{s}'$'s last joint-action. If the two conditions are not satisfied, the transition can not happen because we can neither alter nor erase history (a rule that applies to both RL and life in general). For example, if $\hat{s}^i = \{s_0^i,\mathbf{a}_0^i,s_1^i,\mathbf{a}_1^i,s_2^i\}$ and $\hat{s}^j = \{s_0^j,\mathbf{a}_0^j,s_1^j,\mathbf{a}_1^j,s_2^j, \mathbf{a}_2^j,s_3^j\}$, then,
    $$P_D(\hat{s}^j \mid \hat{s}^i, \mathbf{a}) = P_M(s_3^j \mid s_2^i, \mathbf{a}) \mathbbm{1}(s_0^i = s_0^j)\mathbbm{1}( \mathbf{a}_0^i = \mathbf{a}_0^j)\mathbbm{1}( s_1^i = s_1^j)\mathbbm{1}( \mathbf{a}_1^i = \mathbf{a}_1^j)\mathbbm{1}( s_2^i = s_2^j)\mathbbm{1}( \mathbf{a} = \mathbf{a}_2^j)$$.
    \item The value of a joint-policy $\mathbf{\pi}$ in a Dec-POMDP $D$ is defined as the expected reward obtained by following the joint-policy in $D$:
    $$J(\mathbf{\pi}) = \E_{\hat{s} \sim d_0} \left[V^\pi(\hat{s})\right]  = \E_{s\sim d_0} \left[V^\pi(s)\right].$$ 
\end{itemize}

For the remainder of the derivation, we will overload notation for simplicity. When $P$ takes $\hat{s}$ as input, we mean $P_D$, and when $P$ takes $s$ as input, we mean $P_M$. The same convention follows for $\hat{P}, R$ and $\hat{R}$.

\subsection{Algorithm}
R-MAX \citep{brafman2002r} is a model-based RL (MBRL) algorithm for learning in MDPs. In this work, we adapt R-MAX to Dec-POMDPs (see Algorithm \ref{alg:rmax}). We learn centralized models $\hat{P}_M$, $\hat{R}_M$, $\hat{O}$ using empirical estimates. Using the learned models, we construct an approximate K-known Dec-POMDP $\hat{D}_K$ (see Definition \ref{def:known}). Within $\hat{D}_K$, we evaluate all possible joint-policies $\mathbf{\pi} \in \mathbf{\Pi}$ and choose the best performing one. Our adapted R-MAX algorithm optimistically assigns rewards for all under-visited state action pairs to encourage exploration.

\subsection{Basic Lemmas and Definitions}
The rest of the derivation is heavily based on results from \citeauthor{jiang2020notes} and \citeauthor{strehl2009reinforcement}, but for the first time extends them to the Dec-POMDP setting.

\begin{lemma}[Coin Flip \citep{strehl2009reinforcement}]
\label{lem:coinflip}
Suppose a weighted coin, when flipped, has probability $p>0$ of landing with heads up. Then, for any positive integer $h$ and real number $\delta \in(0,1)$, there exists a number $W=O((h / p) \ln (1 / \delta))$ such that after $W$ tosses, with probability at least $1-\delta$, we will observe $h$ or more heads.
\end{lemma} 
\begin{proof}
    Results follow from the Chernoff-Hoeffding bound.
\end{proof}
\begin{definition}[k-Known Dec-POMDP]
    $D_K$ is the expected version of $\hat{D}_K$ where:
    
    \begin{equation}
    \begin{aligned}
        P_{K}\left(s^{\prime} \mid s, \mathbf{a}\right)=\left\{\begin{array}{lll} P\left(s^{\prime} \mid s, \mathbf{a}\right) & \text { if } (s, \mathbf{a}) \in K \\ \mathbbm{1}\left[s^{\prime}=s\right] & \text { otherwise} & \end{array}\right. &&
        \hat{P}_{K}\left(s^{\prime} \mid s, \mathbf{a}\right)= \begin{cases} \frac{n\left(s, \mathbf{a}, s^{\prime}\right)} {n(s, \mathbf{a})}, & \text { if }(s, \mathbf{a}) \in K \\ \mathbbm{1}\left[s^{\prime}=s\right], & \text { otherwise }\end{cases}
    \end{aligned}
    \end{equation}

    \begin{equation}
        \begin{aligned}
             R_{K}\left(s, \mathbf{a}\right)=\left\{\begin{array}{lll} R\left(s, \mathbf{a}\right) & \text { if } (s, \mathbf{a}) \in K \\ R_\text{max} & \text { otherwise} & \end{array}\right. &&
             \hat{R}_{K}(s, \mathbf{a})= \begin{cases}\frac{\sum_{i}^{n(s,\mathbf{a})}r(s, \mathbf{a})}{n(s,\mathbf{a})}, & \text { if }(s, \mathbf{a}) \in K \\ R_{\max }, & \text { otherwise }\end{cases}
        \end{aligned}
    \end{equation}
    
    \begin{equation}
        \begin{aligned}
             O_{K}\left(s\right)=\left\{\begin{array}{lll} O\left(s\right) & \text { if } (s, \cdot) \in K \\ \text{random observation} & \text { otherwise} & \end{array}\right. &&
             \hat{O}_{K}\left(s\right)=\left\{\begin{array}{lll} \mathbf{o} \text{ where } n\left(s,  \mathbf{o}\right) > 0 & \text { if } (s, \cdot) \in K \\ \text{random observation} & \text { otherwise} & \end{array}\right.
        \end{aligned}
    \end{equation}
    
\begin{fact}
\label{fac:obs}
Because the observation function is deterministic, we only need to see the observation once for a given $(s,\mathbf{a})$ to learn an accurate model. When a given state-action pair $\notin K$, the observation model can be anything since it does not affect the outcome. This is because whatever action the agent performs, the K-known Dec-POMDP transition function will be a self-loop, and the reward will be 1 (i.e. $R_{\max}$). 
\end{fact}

\end{definition}

\begin{table}[!ht]
\begin{center}
\caption{Relationship between $D, D_K$ and $\hat{D}_K$}
\begin{tabular}{|p{5em} | p{1cm} | p{2.9cm} | p{2.9cm}|} 
  \hline
    & $D$ & $D_K$ & $\hat{D}_K$ \\ 
  \hline
  Known ($K$) & $=D$ & $=D$ & $\approx D$ \\ 
  \hline
  Unkown & $=D$ & Self-loop, max reward, random observation & Self-loop, max reward, random observations \\ 
  \hline
\end{tabular}
\end{center}
\label{table:1}
\end{table}

\begin{lemma}[Reward Model Error \citep{strehl2009reinforcement}]
\label{lem:rewardmodel}
Suppose $m$ rewards are drawn independently from the reward distribution, ${R}(s, \mathbf{a})$, for state-action pair $(s, \mathbf{a})$. Let $\hat{R}(s, \mathbf{a})$ be the empirical (maximum-likelihood) estimate of ${R}(s, \mathbf{a}). $Let $\delta_{R}$ be any positive real number less than 1 . Then, with probability at least $1-\delta_{R}$, we have that $|\hat{R}(s, \mathbf{a})-R(s, \mathbf{a})| \leq \epsilon_R$, where
$$
\epsilon_R:=\sqrt{\frac{\ln \left(2 / \delta_{R}\right)}{2 m}} .
$$
\end{lemma}

\begin{proof}
Results follows directly from Hoeffding's bound.
\end{proof}

\begin{lemma}[Transition Model Error \citep{strehl2009reinforcement}] 
\label{lem:transitionmodel}
Suppose $m$ transitions are drawn independently from the transition distribution, ${P}(s, \mathbf{a})$, for state-action pair $(s, \mathbf{a})$. Let $\hat{P}(s, \mathbf{a})$ be the empirical (maximum-likelihood) estimate of ${P}(s, \mathbf{a}) .$ Let $\delta_P$ be any positive real number less than 1 .  Then, with probability at least $1-\delta_{P}$, we have that $\|P(s, \mathbf{a})-\hat{P}(s, \mathbf{a})\|_{1} \leq \epsilon_P$ where

$$\epsilon_P :=  \sqrt{\frac{2}{m} \ln \left[ \left(2^{S}-2\right)/ \delta_P\right]}.$$
\end{lemma}
\begin{proof}
The result follows immediately from an application of Theorem 2.1 of Weissman et al. (2003).
\end{proof}

\begin{corollary}
\label{cor:unionbound}
Suppose for all state-action pairs $(s,\mathbf{a}) \in S \times \mathbf{A}$, $m$ transitions and $m$ rewards are drawn independently from the transition distribution $P(s,\mathbf{a})$ and the reward distribution $R(s,\mathbf{a})$ respectively. Let $\delta$ be any positive real number less than $1$. Then, with probability at least $1-\delta$, for all state-action pairs $(s,\mathbf{a})$, we have $|\hat{R}(s, \mathbf{a})-R(s, \mathbf{a})| \leq \epsilon_R$ and $\|P(s, \mathbf{a})-\hat{P}(s, \mathbf{a})\|_{1} \leq \epsilon_P$, where:
\begin{equation}
    \epsilon_R := \sqrt{\frac{\ln \left(4|S| |\mathbf{A}| / \delta\right)}{2 m}},
\end{equation}
and
\begin{equation}
    \epsilon_P :=  \sqrt{\frac{2}{m} \ln \left[ \left(2^{|S|}-2\right)2|S||\mathbf{A}|/ \delta\right]}.
\end{equation}
\end{corollary}
\begin{proof}
The result is obtained directly by applying the union bound to Lemma \ref{lem:rewardmodel} and Lemma \ref{lem:transitionmodel}, where we set $\delta_P = \delta_R = \frac{\delta}{2|S||\mathbf{A}|}$. We first split the failure probability evenly between the reward estimation events and transition estimation events. This results in the division by a factor of 2. Then, for the transition and the reward, we split $\delta/2$ evenly among all state action pairs. This results in a further division by a factor of $|S||\mathbf{A}|$.
\end{proof}

\begin{fact}[Optimism]
\label{lem:optimism}
By construction of the the Adapted R-MAX algorithm, $\forall \mathbf{\pi}: \mathbf{\hat{S}} \rightarrow \mathbf{A}, \quad J_{D_K}(\mathbf{\pi}) \geqslant J_D(\mathbf{\pi})$.
\end{fact}

\begin{lemma} [Simulation Lemma for Dec-POMDPs, Part 1]
\label{lem:simulation1}
Let $D$ and $\hat{D}$ be two Dec-POMDPs that only differ in $(P, R)$ and $(\hat{P}, \hat{R})$.

Let $\epsilon_{R} \geq \max _{s, \mathbf{a}}|\hat{R}(s, \mathbf{a})-R(s, \mathbf{a})|$ and $\varepsilon_{p} \geq \max _{s, \mathbf{a}} \| \hat{P}(\cdot \mid  s, \mathbf{a}) - P(\cdot \mid s, \mathbf{a})||_{1}$. Then $\forall \mathbf{\pi}: \mathbf{\hat{S}} \rightarrow \mathbf{A}$,

$$\left\|V_{D}^{\pi}-V_{\hat{D}}^{\pi}\right\|_{\infty} \leq \frac{\varepsilon_{R}}{1-\gamma}+\frac{\gamma \epsilon_{P} V_{\max }}{2(1-\gamma)}.$$

\end{lemma}

\begin{proof}

For all $\hat{s} \in \hat{S}$,
\begin{equation}
    \label{eq:sim1}
    \begin{aligned}
    \left| V_{\hat{D}}^{\mathbf{\pi}}(\hat{s})-V_{D}^{\mathbf{\pi}}(\hat{s})\right| & =\left| \hat{R}(\hat{s}, \mathbf{\pi})+ \gamma \left\langle\hat{P}(\hat{s}, \mathbf{\pi}), V_{\hat{D}}^{\mathbf{\pi}}\right\rangle - R(\hat{s}, \mathbf{\pi})-\gamma \left\langle P(\hat{s}, \mathbf{\pi}), V_{D}^{\mathbf{\pi}} \right\rangle \right|  \\
    &= |\hat{R}(\hat{s}, \mathbf{\pi})-R(\hat{s}, \mathbf{\pi})|+\gamma\left|\left\langle\hat{P}(\hat{s}, \mathbf{\pi}), V_{\hat{D}}^{\mathbf{\pi}}\right\rangle-\left\langle P(\hat{s}, \mathbf{\pi}), V_{D}^{\mathbf{\pi}}\right\rangle\right| \quad\quad \text{(re-arranging terms)}\\
    &\leq \varepsilon_{R} + \gamma \left( \left| \left\langle\hat{P}(\hat{s}, \mathbf{\pi}), V_{\hat{D}}^{\mathbf{\pi}}\right\rangle \textcolor{green} {-\left\langle P(\hat{s}, \mathbf{\pi}), V_{\hat{D}}^{\mathbf{\pi}}\right\rangle+\left\langle P(\hat{s}, \mathbf{\pi}), V_{\hat{D}}^\mathbf{\pi} \right\rangle} -\left\langle P(\hat{s}, \mathbf{\pi}), V_{D}^{\mathbf{\pi}}\right\rangle \right| \right) \quad \text{(add \& subtract)}\\
    &\leq \varepsilon_{R}+\gamma\left|<\hat{P}(\hat{s}, \mathbf{\pi})-P(\hat{s}, \mathbf{\pi}), V_{\hat{D}}^{\mathbf{\pi}}>\right|+\gamma\left| \left\langle P(\hat{s}, \mathbf{\pi}), V_{\hat{D}}^{\mathbf{\pi}}-V_{D}^{\mathbf{\pi}}\right\rangle \right| \\
    & \leq \varepsilon_{R}+ \gamma \left|<\hat{P}(\hat{s}, \mathbf{\pi})-P(\hat{s}, \mathbf{\pi}), V_{\hat{D}}^{\mathbf{\pi}}>\right|+\gamma{ \left|\left| V_{\hat{D}^\mathbf{\pi}} - V_D ^\mathbf{\pi} \right|\right|_{\infty}}.
    \end{aligned}
\end{equation}

Because of Fact \ref{fac:obs}, we need not deal with the observation function when writing out the value function, as we simply absorb the observation function as apart of the deterministic joint-policy. 

Since Equation \ref{eq:sim1} holds for all $\hat{s} \in \hat{S}$, we can take the infinite-norm on the left hand side:

\begin{equation}
    \label{eq:simulation_lemma}
    \begin{aligned}
    \left| \left| V_{\hat{D}}^{\mathbf{\pi}}-V_{D}^{\mathbf{\pi}}\right| \right|_{\infty} \leq \varepsilon_{R}+ \gamma \textcolor{blue}{\left|<\hat{P}(\hat{s}, \mathbf{\pi})-P(\hat{s}, \mathbf{\pi}), V_{\hat{D}}^{\mathbf{\pi}}>\right|}+\gamma{ \left|\left| V_{\hat{D}^\mathbf{\mathbf{\pi}}} - V_D ^\mathbf{\pi} \right|\right|_{\infty}}.
    \end{aligned}
\end{equation}

We then expand the blue term as follows:
\begin{equation}
    \label{eq:inner_product}
    \begin{aligned}
  \textcolor{blue}{ \left|\left\langle\hat{P}(\hat{s}, \mathbf{\pi}) - p(\hat{s}, \mathbf{\pi}), V_{\hat{D}}^{\mathbf{\pi}}\right\rangle\right|} &=\left| \left\langle\hat{P}(\hat{s}, \mathbf{\pi})-P(\hat{s}, \mathbf{\pi}), V_{\hat{D}}^{\mathbf{\pi}}-\mathbf{1} \cdot \frac{R_\text{max}}{2(\mathbf{1}-\gamma)}\right\rangle\right| \quad\quad (\text{where $\mathbf{1}$ is a vector of ones} \in \mathbbm{R}^{|S|})\\
  & \leq\|\hat{P}(\hat{s}, \mathbf{\pi})-P(\hat{s}, \mathbf{\pi})\|_{1} \cdot\left\|V_{\hat{D}}^{\mathbf{\pi}}-\mathbf{1}
  \cdot \frac{R_{\max }}{2(1-\gamma)}\right\|_{\infty} \quad\quad (\text{Holder's inequality}) \\
  & \leq \|\hat{P}(s, \mathbf{\pi}(\hat{s})) \mathbbm{1}_{\text{check history}} -P(s, \mathbf{\pi}(\hat{s}))\mathbbm{1}_{\text{check history}}\|_{1}  \cdot\left\|V_{\hat{D}}^{\mathbf{\pi}}-\mathbf{1} \cdot \frac{R_{\max }}{2(1-\gamma)}\right\|_{\infty} \quad\quad (\text{by definition of } P) \\
  & = \|\hat{P}(s, \mathbf{\pi}(\hat{s})) -P(s, \mathbf{\pi}(\hat{s}))\|_{1}  \cdot\left\|V_{\hat{D}}^{\mathbf{\pi}}-\mathbf{1} \cdot \frac{R_{\max }}{2(1-\gamma)}\right\|_{\infty}  \\
  & \leq \epsilon_P \cdot \frac{R_\text{max}}{2(1-\gamma)} \\
  & = \epsilon_P \cdot \frac {V_\text{max}}{2}.
    \end{aligned}
\end{equation}

In Equation \ref{eq:inner_product}, the step in line 1 shifts the range of $V$ from $[0, \frac{R_\text{max}}{(1-\gamma)}]$ to $[-\frac{R_\text{max}}{2(1-\gamma)}, \frac{R_\text{max}}{2(1-\gamma)}]$ to obtain a tighter bound by a factor of 2. The equality in line 1 holds because of the following, where $C$ is any constant:

\begin{equation}
\label{eq:intermediate}
\begin{aligned}
     \langle \hat{P} - P, C \cdot \mathbf{1} \rangle &= C \langle \hat{P} - P, \mathbf{1}
    \rangle \\
    &= C \left(\langle \hat{P}, \mathbf{1} \rangle - \langle P, \mathbf{1} \rangle \right) \\
    &= C (1-1) \quad \quad \text{because $P$ and $\hat{P}$ are probability distributions} \\
    &= 0
    \end{aligned}    
\end{equation}
From equation \ref{eq:intermediate}, we observe the equality in line 1 holds:
\begin{equation}
    \begin{aligned}
    \langle \hat{P} - P, V - C \cdot \mathbf{1} \rangle &= \langle \hat{P} - P, V \rangle - \langle \hat{P} - P, C \cdot \mathbf{1} \rangle \\
    &=  \langle \hat{P} - P, V \rangle - 0 \\
    & =  \langle \hat{P} - P, V\rangle.
    \end{aligned}
\end{equation}

Finally, we plug equation \ref{eq:inner_product} into equation \ref{eq:simulation_lemma} to obtain the bound.

\begin{equation}
    \begin{aligned}
    \left|\left| V_{\hat{D}}^{\mathbf{\pi}}-V_{D}^{\mathbf{\pi}}\right|\right|_{\infty} & \leq \epsilon_R + \gamma \epsilon_P \cdot \frac {V_\text{max}}{2} + \gamma \left|\left|  V_{\hat{D}}^{\mathbf{\pi}}-V_{D}^{\mathbf{\pi}}\right|\right|_{\infty} \\
    &= \frac{\epsilon_R}{1-\gamma} +  \frac {\gamma \epsilon_P V_\text{max}}{2(1-\gamma)}.
    \end{aligned}
\end{equation}

\end{proof}

\begin{lemma} [Simulation Lemma for Dec-POMDPs, Part 2]
\label{lem:simulation2}
Let $D$ and $\hat{D}$ be two Dec-POMDPs that only differ in $(P, R)$ and $(\hat{P}, \hat{R})$, then the following holds:

Let $\epsilon_{R} \geq \max _{s, \mathbf{a}}|\hat{R}(s, \mathbf{a})-R(s, \mathbf{a})|$ and $\varepsilon_{p} \geq \max _{s, \mathbf{a}} \| \hat{P}(\cdot \mid  s, \mathbf{a}) - P(\cdot \mid s, \mathbf{a})||_{1}$, then $\forall \mathbf{\pi}: \mathbf{\hat{S}} \rightarrow \mathbf{A}$,

$$\left\|V_{D}^{*}-V_{\hat{D}}^{*}\right\|_{\infty} \leq \frac{\varepsilon_{R}}{1-\gamma}+\frac{\gamma \epsilon_{P} V_{\max }}{2(1-\gamma)}.$$
\end{lemma}

\begin{proof}
Let $\mathcal{T}_{D}, \mathcal{T}_{\hat{D}}$ be the Bellman update operator of $D$ and $\hat{D}$ respectively.
\begin{equation}
\begin{aligned}
\label{eq:10}
\left\|V_{D}^{*}-\mathcal{T}_{\hat{D}} V_{D}^{*}\right\|_{\infty}  = &\left\|\mathcal{T}_{D} V_{D}^{*}-\mathcal{T}_{\hat{D}} V_{D}^{*}\right\|_{\infty} \\
=& \max _{\hat{s}, \mathbf{a} \in  \hat{S} \times \mathbf{A}}\left| R(\hat{s},\mathbf{a}) + \gamma \mathbb{E}_{\hat{s}^{\prime} \sim P(\hat{s}, \mathbf{a})}\left[V_{D}^{*}\left(s^{\prime}\right)\right]- \hat{R}(\hat{s},\mathbf{a}) - \gamma \mathbb{E}_{\hat{s}^{\prime} \sim \hat{P}(\hat{s}, \mathbf{a})}\left[V_{D}^{*}\left(s^{\prime}\right)\right]\right| \\
=&  \max _{\hat{s}, \mathbf{a} \in  \hat{S} \times \mathbf{A}}\left|  R(\hat{s},\mathbf{a})  - \hat{ R}(\hat{s},\mathbf{a}) \right | + \max _{\hat{s}, \mathbf{a} \in  \hat{S} \times \mathbf{A}}\left|\gamma \mathbb{E}_{\hat{s}^{\prime} \sim P(\hat{s}, \mathbf{a})}\left[V_{D}^{*}\left(\hat{s}^{\prime}\right)\right] - \gamma \mathbb{E}_{\hat{s}^{\prime} \sim \hat{P}(\hat{s}, \mathbf{a})}\left[V_{D}^{*}\left(s^{\prime}\right)\right]\right| \\
=& \epsilon_R + \gamma \max _{\hat{s}, a \in \hat{S} \times \mathbf{A}}\left\langle P(\hat{s}, \mathbf{a})-\hat{P}(\hat{s}, \mathbf{a}), V_{D}^{*}-V_{\max } / 2 \cdot \mathbf{1}\right\rangle \quad \quad \text{where $\mathbf{1}$ is an all ones vector} \\
\leq & \epsilon_R + \gamma \max _{\hat{s}, a \in \hat{S} \times \mathbf{A}}\left\|P(\hat{s}, \mathbf{a})-\hat{P}(\hat{s}, \mathbf{a})\right\|_{1}\left\|V_{D}^{*}-V_{\max } / 2 \cdot \mathbf{1}\right\|_{\infty} \\ 
\leq & \epsilon_R +  \gamma \epsilon_P \cdot V_{\max } / 2.
\end{aligned}
\end{equation}

Therefore, 

\begin{equation}
\begin{aligned}
\left\|V_{D}^{*}-V_{\hat{D}}^{*}\right\|_{\infty} &=\left\|V_{D}^{*}-\mathcal{T}_{\hat{D}} V_{D}^{*}+\mathcal{T}_{\hat{D}} V_{D}^{*}-\mathcal{T}_{\hat{D}} V_{\hat{D}}^{*}\right\|_{\infty} \quad \quad \text{(add and subtract)}\\
& \leq \epsilon_R +  \gamma \epsilon_P \cdot V_{\max } / 2 +\left\|\mathcal{T}_{\hat{D}} V_{D}^{*}-\mathcal{T}_{\hat{D}} V_{\hat{D}}^{*}\right\|_{\infty} \quad \quad \text{(by equation \ref{eq:10})}\\
& \leq \epsilon_R +  \gamma \epsilon_P \cdot V_{\max } / 2 +\gamma\left\|V_{D}^{*}-V_{\hat{D}}^{*}\right\|_{\infty}   \quad \quad \text{(by property of $V^*$)}.
\end{aligned}
\end{equation}

\end{proof}

\begin{corollary}
\label{cor:simulation}
Following from the two Simulation Lemmas for Dec-POMDPs, we have the following:
\begin{enumerate}
    \item $\forall \mathbf{\pi}, \left|J_{\hat{D}_{k}}(\mathbf{\pi})-J_{D_ k}(\mathbf{\pi})\right| \leq  \frac{\varepsilon_{R}}{1-\gamma}+\frac{\gamma \epsilon_{P} V_{\max }}{2(1-\gamma)}.$
    \item $\forall \mathbf{\pi}, \left|J_{D_K}(\mathbf{\pi}^*_{D_K}) - J_{\hat{D}_K}(\mathbf{\pi}^*_{\hat{D}_K})\right| \leq  \frac{\varepsilon_{R}}{1-\gamma}+\frac{\gamma \epsilon_{P} V_{\max }}{2(1-\gamma)}.$
\end{enumerate}
\end{corollary}
\begin{proof}
We give a proof for (2) as follows:

By definition, $J(\mathbf{\pi}) = \E_{s \sim d_0}[V^\mathbf{\pi}] = \langle d_0, V^\mathbf{\pi} \rangle$.

\begin{equation}
\begin{aligned}
 \left|J_{D_K}(\mathbf{\pi}^*_{D_K}) - J_{\hat{D}_K}(\mathbf{\pi}^*_{\hat{D}_K})\right| &= \left| \langle d_0, V^*_{D_K} \rangle  - \langle d_0, V^*_{\hat{D}_K} \rangle \right| \\
 &= \left| \langle d_0, V^*_{D_K} - V^*_{\hat{D}_K}   \rangle \right| \\
 & \leq \langle d_0, \mathbf{1} \cdot \left\|V^*_{D_K} - V^*_{\hat{D}_K} \right\|_\infty \rangle \quad \quad \text{where $\mathbf{1}$ is the all ones vector}\\
 & = \left\|V^*_{D_K} - V^*_{\hat{D}_K} \right\|_\infty \langle d_0, \mathbf{1} \rangle \\
 & = \left\|V^*_{D_K} - V^*_{\hat{D}_K} \right\|_\infty \quad \quad \text{(because $d_0$ is a probability distribution)} \\
 & \leq \frac{\epsilon_R}{1-\gamma} +  \frac {\gamma \epsilon_P V_\text{max}}{2(1-\gamma)} \quad \quad \text{(by lemma \ref{lem:simulation2})}.
\end{aligned}
\end{equation}

The proof for (1) follows from the same steps as above using Lemma \ref{lem:simulation1}.

\end{proof}

\begin{lemma} [Induced Inequality]
\label{lem:inequality}
Suppose Dec-POMDPs $D$ and $D_{k}$ has centralized transition/reward/observation functions that agree exactly on $K \subseteq S \times \mathbf{A}$. Let escape $_{K}(\hat{\tau})$ be 1 if the trajectory $\hat{\tau}$ contains at least one $(s, \mathbf{a}) \notin K$, and 0 otherwise. $\forall \mathbf{\pi}: \mathbf{\hat{S}} \rightarrow \mathbf{A}$,
$$
\left|J_{D}(\mathbf{\pi})-J_{D_K}(\mathbf{\pi})\right| \leq V_{\max } \cdot \mathbb{P}_{D}\left[\text {escape}_{K}(\hat{\tau}) \mid \mathbf{\pi}\right] .
$$
\end{lemma}
Proof. Let $R_{D}(\hat{\tau})$ denote the sum of discounted rewards in $\hat{\tau}$ according to the reward function of $D$. We can write $V_{D}^{\mathbf{\pi}}=\sum_{\hat{\tau}} \mathbb{P}_{D}[\hat{\tau} \mid \mathbf{\pi}] R_{D}(\hat{\tau})$ (and similarly for $D_K$ ). For $\hat{\tau}$ such that $\text{escape}_{K}(\hat{\tau})=1$, define pre$_{K}(\hat{\tau})$ as the prefix of $\hat{\tau}$ where every state-action pair in $\hat{\tau}$ is also in $K$, and that the next $\hat{s}$ contains one that escapes. Similarly define suf$_{K}(\hat{\tau})$ as the remainder of the episode. See Figure \ref{fig:presuf} for an illustration of an example. Let $R\left(\operatorname{pre}_{K}(\hat{\tau})\right)$ be the sum of discounted rewards within the prefix (or suffix), and $\mathbb{P}_{D_K}\left[\operatorname{pre}_{K}(\hat{\tau}) \mid \mathbf{\pi}\right]$ be the marginal probability of the prefix (or suffix) assigned by $D_K$ under the joint-policy $\mathbf{\pi}$. Because $J_{D_K} \geq J_{D}$ by optimism (Fact \ref{lem:optimism}), its sufficient to upper bound $J_{D_K}(\mathbf{\pi})-J_{D}(\mathbf{\pi})$. 

\begin{figure}[h!]
    \centering
    \includegraphics[width=0.60\columnwidth]{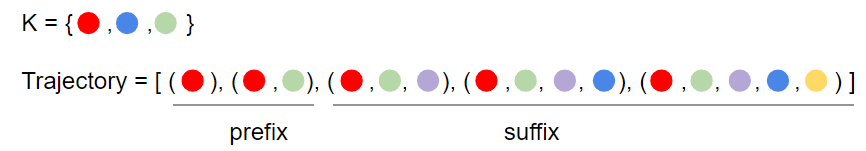}
    \caption{An example of partitioning a trajectory $\hat{\tau}$ into the prefix and the suffix. Different colours of a circle denote different state-action pairs, and each bracket in the trajectory is a $\hat{s}$. Because purple in $\hat{s}_2$ is the first time the trajectory sees a state-action pair not in the set $K$, so prefix$=\left[\hat{s}_0, \mathbf{a}_0,  \hat{s}_1,\mathbf{a}_1\right]$, and suffix$=\left[\hat{s}_2,\mathbf{a}_2, \hat{s}_3\right]$.}
    \label{fig:presuf}
\end{figure}

First, we give an upper bound on $J_{D_K}$ using the fact that the cumulative discounted reward is upper-bounded by $V_\text{max}$:
\begin{equation}
    \begin{aligned}
        J_{D_K}(\mathbf{\pi}) &=  \sum_{\hat{\tau} \text {:escape }_{K}(\hat{\tau})=1} \mathbb{P}_{D_K}[\hat{\tau} \mid \mathbf{\pi}] R_{D_K}(\hat{\tau}) + \sum_{\hat{\tau} \text {:escape}_{K}(\hat{\tau})=0} \mathbb{P}_{D_K}[\hat{\tau} \mid \mathbf{\pi}] R_{D_K}(\hat{\tau}) \\
        &= \sum_{\hat{\tau} \text {:escape}_{K}(\hat{\tau})=1} \mathbb{P}_{D_K}[\hat{\tau} \mid \mathbf{\pi}]\left(R_{D_K}\left(\operatorname{pre}_{K}(\hat{\tau})\right)+R_{D_K}\left(\operatorname{suf}_{K}(\hat{\tau})\right)\right)+\sum_{\hat{\tau} \text {:escape}_{K}(\hat{\tau})=0} \mathbb{P}_{D_K}[\hat{\tau} \mid \mathbf{\pi}] R_{D_K}(\hat{\tau}) \\
        & \leq \sum_{\hat{\tau} \text {:escape}_{K}(\hat{\tau})=1} \mathbb{P}_{D_K}[\hat{\tau} \mid \mathbf{\pi}]\left(R_{D_K}\left(\operatorname{pre}_{K}(\hat{\tau})\right)+V_{\max }\right)+\sum_{\hat{\tau} \text {:escape}_{K}(\hat{\tau})=0} \mathbb{P}_{D_K}[\hat{\tau} \mid \mathbf{\pi}] R_{D_K}(\hat{\tau})\\ 
        & \leq \sum_{\text {pre}_{K}(\hat{\tau})} \mathbb{P}_{D_K}\left[\operatorname{pre}_{K}(\hat{\tau}) \mid \mathbf{\pi}\right]\left(R_{D_K}\left(\operatorname{pre}_{K}(\hat{\tau})\right)+V_{\max }\right)+\sum_{\hat{\tau}: \text {escape}_{K}(\hat{\tau})=0} \mathbb{P}_{D_K}[\hat{\tau} \mid \mathbf{\pi}] R_{D_K}(\hat{\tau}).
    \end{aligned}
\end{equation}
First line uses the fact that trajectories falls in either category - those that contain an escape, and those that do not. The second line follows from the fact that we can split each trajectory containing at least one escape into a prefix and suffix. The third line uses the fact that the prefix trajectory's value is the same as the true Dec-POMDP $D$, and that the value of the suffix trajectory cannot be greater than $V_{\max}$. The last step uses the fact that for any $\hat{\tau}$ that shares the same pre$_{K}(\hat{\tau})$, we can combine their probabilities (because $R\left(\right.$pre$\left._{K}(\hat{\tau})\right)+V_{\max }$ does not depends on the suffix), and we get the marginal probability of the prefix. 

Similarly, we give a lower bound on $J_{D_K}$, this time using the fact that the cumulative discounted reward is lower-bounded by $0$:
\begin{equation}
    J_{D}(\mathbf{\pi}) \geq \sum_{\operatorname{pre}_{K}(\hat{\tau})} \mathbb{P}_{D}\left[\operatorname{pre}_{K}(\hat{\tau}) \mid \mathbf{\pi}\right] R_{D}\left(\operatorname{pre}_{K}(\hat{\tau})\right)+\sum_{\tau \text {:escape}_{K}(\hat{\tau})=0} \mathbb{P}_{D}[\tau \mid \mathbf{\pi}] R_{D}(\hat{\tau}).
\end{equation}

By the definition of $D_K$ and $D$, we have the following:
\begin{enumerate}
    \item $\mathbb{P}_{D_K}\left[\operatorname{pre}_{K}(\hat{\tau}) \mid \mathbf{\pi}\right] = \mathbb{P}_{D}\left[\operatorname{pre}_{K}(\hat{\tau}) \mid \mathbf{\pi}\right].$
    \item $R_{D_K}\left(\operatorname{pre}_{K}(\hat{\tau})\right) = R_{D}\left(\operatorname{pre}_{K}(\hat{\tau})\right).$
    \item $\forall \text{escape}_K(\hat{\tau})=0, \mathbb{P}_{D_K}\left[\hat{\tau}\mid \mathbf{\pi}\right] = \mathbb{P}_{D}\left[\hat{\tau}\mid \mathbf{\pi}\right].$
    \item $\forall \text{escape}_K(\hat{\tau})=0, R_{D_K}(\hat{\tau}) = R_{D}(\hat{\tau}).$
\end{enumerate}

Combining the above, we obtain the bound.
\begin{equation}
\begin{aligned}
    & J_{D_K}(\mathbf{\pi}) - J_{D}(\mathbf{\pi}) \\
    & \leq \text{UpperBound}(J_{D_K}(\mathbf{\pi})) - \text{LowerBound}(J_D(\mathbf{\pi}))\\
    & = \sum_{\text {pre}_{K}(\hat{\tau})} \mathbb{P}_{D_K}\left[\operatorname{pre}_{K}(\hat{\tau}) \mid \mathbf{\pi}\right]\left(R_{D_K}\left(\operatorname{pre}_{K}(\hat{\tau})\right)+V_{\max }\right)+\sum_{\hat{\tau}: \text {escape}_{K}(\hat{\tau})=0} \mathbb{P}_{D_K}[\hat{\tau} \mid \mathbf{\pi}] R_{D_K}(\hat{\tau}) \\ 
& \quad \quad \quad \quad - \sum_{\operatorname{pre}_{K}(\hat{\tau})} \mathbb{P}_{D}\left[\operatorname{pre}_{K}(\hat{\tau}) \mid \mathbf{\pi}\right] R_{D}\left(\operatorname{pre}_{K}(\hat{\tau})\right) - \sum_{\hat{\tau} \text {:escape}_{K}(\hat{\tau})=0} \mathbb{P}_{D}[\hat{\tau} \mid \mathbf{\pi}] R_{D}(\hat{\tau}) \\ 
    &= \sum_{\text {pre}_{K}(\hat{\tau})} \mathbb{P}_{D_K}\left[\operatorname{pre}_{K}(\hat{\tau}) \mid \mathbf{\pi}\right]\left(R_{D_K}\left(\operatorname{pre}_{K}(\hat{\tau})\right)+V_{\max }\right) -  \sum_{\text {pre}_{K}(\hat{\tau})} \mathbb{P}_{D}\left[\operatorname{pre}_{K}(\hat{\tau}) \mid \mathbf{\pi}\right]\left(R_{D}\left(\operatorname{pre}_{K}(\hat{\tau})\right)\right) \quad\quad \text{(by (3) and (4))} \\
    &=  \sum_{\text {pre}_{K}(\hat{\tau})}\mathbb{P}_{D}\left[\operatorname{pre}_{K}(\hat{\tau}) \mid \mathbf{\pi}\right] V_{\max} \quad\quad \text{(by (1) and (2))} \\
    &= \mathbb{P}_{D} \left[ \text{escape} (\hat{\tau}) \mid \mathbf{\pi}\right]  V_{\max}.
\end{aligned}    
\end{equation}

\subsection{Sample Complexity Analysis}
\begin{theorem}
Suppose that $0 \leq \varepsilon<1$ and $0 \leq \delta<1$ are two real numbers and $D$ is any Dec-POMDP. There exists inputs $m=m\left(\frac{1}{\varepsilon}, \frac{1}{\delta}\right)$ and $\varepsilon$, satisfying $m\left(\frac{1}{\varepsilon}, \frac{1}{\delta}\right)=O\left(\frac{(S+\ln (S A / \delta)) V_{\max }^{2}}{\varepsilon^{2}(1-\gamma)^{2}}\right)$ such that if the Adapted R-MAX is executed on $D$ with inputs $m$ and $\varepsilon$, then the following holds: Let $\mathbf{\pi}^*_{\hat{D}_K}$ denote the Adapted R-MAX 's joint-policy. With probability at least $1-\delta$, $J_D(\mathbf{\pi}^*) - J_D(\mathbf{\pi}^*_{\hat{D}_K}) \leq 2\epsilon$ is true for all but
$$O\left(\frac{|S||\mathbf{A}|}{(1-\gamma)^2\epsilon^3} \left(|S| + \ln(\frac{|S||\mathbf{A}|}{\delta})\right)V_{\max}^3 \ln \frac{1}{\delta}\right)$$ episodes.

\end{theorem}



\begin{proof}
\begin{equation}
\label{eq:16}
\begin{aligned}
    J_D(\mathbf{\pi}^*_D) - J_D(\mathbf{\pi}^*_{\hat{D}_K}) & \leq J_{D_K}(\mathbf{\pi}^*_D) - J_D(\mathbf{\pi}^*_{\hat{D}_K}) \quad\quad \text{(by optimism in lemma \ref{lem:optimism})}\\
    & \leq J_{D_K}(\mathbf{\pi}^*_{D_K}) - J_D(\mathbf{\pi}^*_{\hat{D}_K}) \quad\quad \text{(by definition of greedy policy)}\\
    & \leq J_{\hat{D}_K}(\mathbf{\pi}^*_{\hat{D}_K}) + \frac{\varepsilon_{R}}{1-\gamma}+\frac{\gamma \epsilon_{P} V_{\max }}{2(1-\gamma)} - J_D(\mathbf{\pi}^*_{\hat{D}_K}) \quad\quad \text{(by (2) in Corollary \ref{cor:simulation} )}\\
    & \leq J_{D_K}(\mathbf{\pi}^*_{\hat{D}_K}) +2 \frac{\varepsilon_{R}}{1-\gamma}+\frac{\gamma \epsilon_{P} V_{\max }}{(1-\gamma)} - J_D(\mathbf{\pi}^*_{\hat{D}_K}) \quad\quad \text{(by (1) in orollary \ref{cor:simulation} )}\\
    &\leq 2\frac{\varepsilon_{R}}{1-\gamma}+\frac{\gamma \epsilon_{P} V_{\max }}{(1-\gamma)} + \mathbb{P}_{D}\left[ \text{escape} (\hat{\tau}) \mid \mathbf{\pi}\right]  V_{\max} \quad\quad \text{(by Lemma \ref{lem:inequality} )}.
\end{aligned}
\end{equation}

In corollary \ref{cor:simulation}, we essential say $|J_1 - J_2| \leq \Delta$. We then have the following upper bound $J_1 \leq J_2 + \Delta$. This is because either of the two cases happens:
\begin{enumerate}
    \item Case 1, $J_1 \geq J_2$: 
    $$
    \begin{aligned}
     & |J_1 - J_2|  \leq \Delta \\
     \Rightarrow & J_1 - J_2  \leq \Delta \\
     \Rightarrow & J_1  \leq J_2 + \Delta
    \end{aligned}
    $$
    \item Case 2, $J_1 < J_2 \Rightarrow  J_1  \leq  J_2 + \Delta $ because $\Delta \leq 0$.
\end{enumerate}
In line 3 and 4 of equation \ref{eq:16}, we essential just apply the upper bound $J_1 \leq J_2 + \Delta$.

To choose $m$ large enough such that $2\frac{\varepsilon_{R}}{1-\gamma}+\frac{\gamma \epsilon_{P} V_{\max }}{(1-\gamma)} \leq \epsilon$ with high probability, by Corollary \ref{cor:unionbound}, we need the following two conditions to hold:
\begin{equation}
\begin{aligned}
    \sqrt{\frac{\ln \left(4|S| |\mathbf{A}| / \delta_{R}\right)}{2 m}} \frac{2}{1-\gamma} &\leq \epsilon \\
\end{aligned}
\end{equation}

\begin{equation}
     \sqrt{\frac{2}{m} \ln \left[ \left(2^{S}-2\right)2|S||\mathbf{A}|/ \delta_P\right]} \frac{\gamma V_{\max}}{(1-\gamma)} \leq \epsilon
\end{equation}

Solving for $m$, we see that when $m \geq C V_{\max }^{2}\left(\frac{S+\ln (S \mathbf{A} / \delta)}{\varepsilon^{2}(1-\gamma)^{2}}\right) $, both conditions are satisfied ($C$ is a constant).

There are two possible cases for the remaining term $\mathbb{P}_{D}\left[ \text{escape} (\hat{\tau}) \mid \mathbf{\pi}\right]$:
\begin{enumerate}
    \item Suppose that $\mathbb{P}_{D}\left[ \text{escape} (\hat{\tau}) \mid \mathbf{\pi}\right] < \frac{\epsilon}{V_{\max}}$, then the joint-policy is 2$\epsilon$-optimal.
    \item Suppose $\mathbb{P}_{D}\left[ \text{escape} (\hat{\tau}) \mid \mathbf{\pi}\right] \geq \frac{\epsilon}{V_{\max}}$, then the following occurs: Because successful exploration occurs at most $m|S| |\mathbf{A}|$ times, by Lemma \ref{lem:coinflip}, we need $W$ episodes with attempts of exploration, where
    \begin{equation}
        \begin{aligned}
         W &= O\left(\frac{m|S||\mathbf{A}|} {\mathbb{P}_{D}\left[ \text{escape} (\hat{\tau}) \mid \mathbf{\pi}\right]} \ln \frac{1}{\delta}\right) \\
         &= O\left(\frac{|S||\mathbf{A}|}{(1-\gamma)^2\epsilon^3} \left(|S| + \ln(\frac{|S||\mathbf{A}|}{\delta})\right)V_{\max}^3 \ln \frac{1}{\delta}\right).
        \end{aligned}
    \end{equation}
\end{enumerate}

\end{proof}

\section{Experiment Details}

\subsection{Policy Optimization} Implementation of model-free baseline and MARCO is based on the \texttt{pyMARL} codebase~\cite{samvelyan19smac}.

In each experiment, all agents share the same network. To distinguish between each agent, a one-hot vector of agent ID is concatenated on to each agent's observation. The architecture of all agent networks is a DRQN with a recurrent layer consisting of a fully-connected layer, followed by a GRU with hidden dimension of 64 units, then another fully-connected layer. We use a $\gamma=0.99$ for all experiments. The replay buffer contains the most recent 5000 episodes. During training, we sample a batch of 32 episodes uniformly from the buffer each time, and the target network is updated every 200 episodes. When training the agent's policy, we encourage exploration by using an $\epsilon$-greedy policy, where $\epsilon$ anneals linearly from $1.0$ to $0.05$ over 100k training timestamps. after 100k timesteps, $\epsilon$ remains constant at $0.05$. All neural nets during policy optimization is trained using RMSprop with learning rate of $1 \times 10^{-4}$. 

The exploration policy is trained in the same manner as described above. Whenever a new model is learned, the exploration policy is trained for 100k timesteps within the model in the switch tasks, and for 50k in the MPE task. The input to the exploration policy uses the state instead of the agent's observation.

\subsection{Data Collection}
The dataset for model learning is always initialized with data collected using a randomly behaving policy. After the initialization, we use a $\epsilon$-greedy data-collection policy, where $\epsilon$ is constant with value of $0.1$. 

\subsection{Model Learning}
The dynamics model for both switch tasks is an auto-regressive model. An encoder with a single hidden layer of 500 units encodes the past state-action pair. The encoded state-action pair is then used to generate the next state using a recurrent layer composed of a fully connected layer, followed by a GRU with hidden dimension of 500 units, and two more fully connected layers. The dynamics model for the MPE task uses a fully connected neural network of two hidden layers of 500 units to predict the categorical state features and continuous state features separately. All other model components are parameterized by neural networks with two fully connected layers of 500 units. 

All model components are trained with the Adam optimizer, learning rate of $0.001$, and batch size of 1000. To prevent overfitting, we use early stopping with a maximum of 700 epochs. The training and validation dataset is split between 3:7 respectively. 

\paragraph{Switch Tasks} We learn an ensemble of 5 models over the dynamics model, and 1 model for all other models components.  We also learn a model for the available actions function to narrow down the joint-policy search space. At each timestep, the available actions function outputs a set of available actions.
$$\text{Available actions model}: \qquad  {P_\mathrm{avail}}_\psi(\{a\}_{j=1:N}|s_t)$$ The model-free baseline have access to the ground truth available actions function during both training and testing for a fair comparison.

\paragraph{MPE} We learn an ensemble of 3 models over the dynamics model and the reward model.

\subsection{Compute}
All experiments are done on a GPU of either NVIDIA T4 or NVIDIA P100. Experiments took $< 12$ hours to run.

\end{document}